\documentclass[lettersize,journal]{IEEEtran}

\newcommand*{\QEDA}{\hfill \ensuremath{\blacksquare}}
\usepackage[utf8x]{inputenc}
\usepackage{scalefnt}
\usepackage{stackengine}

\usepackage{bbm}
\usepackage{xcolor}
\usepackage{cite}
\usepackage[ruled,vlined]{algorithm2e}
\SetAlCapHSkip{0pt}
\usepackage{footnote}
\makesavenoteenv{tabular}
\makesavenoteenv{table}
\usepackage{listings}
\usepackage{amsmath}
\usepackage{amsthm,amssymb,graphicx,multirow,amsmath,color,amsfonts}

\def\nb0{{\mathbf{0}}}
\def\nb1{{\mathbf{1}}}

\newtheorem{lemma}{Lemma}

\newtheorem{theorem}{Theorem}

\newtheorem{cor}{Corollary}

\newtheorem{assumption}{Assumption}
	
\usepackage{dsfont}
\allowdisplaybreaks 
\usepackage{color} 
\definecolor{mygreen}{RGB}{28,172,0} 
\usepackage{tikz}
\usepackage{lipsum,adjustbox}
\usetikzlibrary{calc,positioning}
\usepackage{pst-tree,array} 
\definecolor{mylilas}{RGB}{170,55,241}
\usepackage{graphicx}
 \usepackage{multirow}
\usepackage{enumitem}
\usepackage{booktabs}
 
\usepackage{verbatim}
\usepackage{subcaption}

\title{Multi-agent Off-policy Actor-Critic Reinforcement Learning for Partially Observable Environments}
\author{Ainur Zhaikhan and Ali H. Sayed

\thanks{A. Zhaikhan and A. H. Sayed are with the Adaptive
Systems Laboratory, \'Ecole Polytechnique F\'ed\'erale de Lausanne (EPFL),
CH-1015, Switzerland. Emails: ainur.zhaikan@epfl.ch and ali.sayed@epfl.ch}}
\begin{document}
\maketitle

\begin{abstract}
This study proposes the use of a social learning method to estimate a global state within a multi-agent off-policy actor-critic algorithm for reinforcement learning (RL) operating in a partially observable environment. We assume that the network of agents operates in a fully-decentralized manner, possessing the capability to exchange variables with their immediate neighbors. The proposed design methodology is supported by an analysis demonstrating that the difference between final outcomes, obtained when the global state is fully observed versus estimated through the social learning method, is $\varepsilon$-bounded when an appropriate number of iterations of social learning updates are implemented. Unlike many existing dec-POMDP-based RL approaches, the proposed algorithm is suitable for model-free multi-agent reinforcement learning as it does not require knowledge of a transition model. Furthermore, experimental results illustrate the efficacy of the algorithm and demonstrate its superiority over the current state-of-the-art methods.
\end{abstract}
\section{Introduction}
\label{sec:Intro}
 Collaborative multi-agent reinforcement learning is a prominent area of research due to its potential to achieve outcomes that are either inefficiently achieved or altogether unattainable by individual agents. However, this collaborative approach introduces its own set of complexities and constraints. One fundamental challenge is the issue of partial observability. When multiple agents are deployed across an environment, they may only have access to limited information, observing only specific portions of the overall environment. \par

Most traditional adaptations of single-agent reinforcement learning (RL) to multi-agent scenarios assume full observability of the state variable and require joint knowledge among agents. Some MARL algorithms adopt centralized training and decentralized execution (CTDE) paradigms, such as MADDPG \cite{MADPPG}, QMIX \cite{QMIX}, and MAVEN \cite{ MAVEN}, or Independent Learning (IL) paradigms \cite{IL1}, which still involve restrictive assumptions that may not be applicable to complex scenarios. In CTDE, each agent can access the actions and observations of all other agents during the learning phase, which are required for updating variables of interest. However, during the execution phase, the policies followed by each agent are private and based solely on their local observations. In Independent Learning, agents ignore the learning of other agents and treat the effects of other agents as environmental changes, potentially undermining the stability of the overall model.\par 
The Decentralized Partially Observable Markov Decision Process (Dec-POMDP) setting, which perfectly describes partially observable MARL scenarios, is proven to be an NEXP-hard problem and lacks a formal solution. Consequently, understanding and addressing these partial observability scenarios is crucial for the practical application of multi-agent reinforcement learning. \par 

In this work, we extend the multi-agent off-policy actor-critic (MAOPAC) algorithm \cite{SUTTLE20201549}, originally designed for fully observable environments, to handle partially observable settings. The primary reason for focusing on this MARL setting is that off-policy learning \cite{Off_Policy1, Off_Policy2, Off_Policy3, Off_Policy4} provides a more general setting compared to its on-policy counterpart, as it allows the behavioral policy to differ from the estimated target policy. Moreover, off-policy algorithms offer the flexibility to adjust agents' behavior, thereby enhancing the efficiency of their exploration strategies. Additionally, MAOPAC is an algorithm with theoretical convergence guarantees. Consequently, by establishing the convergence of our proposed method in partially observed environments, we ensure that our results also guarantee convergence to the optimal solution. \par 
Achieving full decentralization in actor-critic approaches is inherently challenging because they depend on shared variables for learning \cite{MA_POMDP} and also suffer from partial state observability. To overcome this challenge, we propose leveraging {\em social learning} strategies \cite{SL_1,SL2,SL3} to estimate the global state in a {\em fully decentralized} manner. Under these strategies,  agents  estimate belief vectors using local observations and then iteratively diffuse these estimates to their immediate neighbors. Existing works that offer a fully decentralized solution often rely on neural network implementations which can be complex and challenging to analyze. Furthermore, unlike many existing algorithms, the proposed method does not necessitate transition models for state estimation, thereby classifying it as a {\em model-free} reinforcement learning algorithm. This attribute enhances its convenience and practicality for real-world applications. 
 The proposed method is supported with  theoretical guarantees. Specifically, we derive conditions for estimating the global state, ensuring that the ultimate error in the policy parameter estimation is bounded by $\varepsilon$. We also provide specific insights tailored to a Boltzmann policy model, providing a more comprehensive understanding of our proposed framework. Through empirical evaluation and analysis, we demonstrate the effectiveness and robustness of our approach in addressing the challenges posed by partially observable multi-agent environments. Additionally, we benchmark our method against the zero-th order policy optimization (ZOPO) approach, which, to the best of our knowledge, stands as the most recent model-free method tailored for partially observable multi-agent settings. Through experimental validation, we illlustrate the superiority of our proposed algorithm in terms of performance. Furthermore, our method offers greater convenience in implementation, thereby presenting a compelling advantage over existing alternatives.\par
 The organization of the paper is as follows. Section \ref{sec:state of the art} gives an overview of the current state of the art. Section \ref{sec: Preliminary} provides preliminary information relevant to the algorithm. The details of the proposed algorithm are given in Section \ref{sec:algorithm}.  Analytical verification of the algorithm is conducted in Section \ref{sec:Theory}, followed by empirical validation in Section \ref{sec:experiment}. Lastly, Section \ref{sec:Conclusion} presents the concluding remarks.
\section{State of the art}
\label{sec:state of the art}
\subsection{Settings} Multi-agent reinforcement learning within a partially observable environment presents a complex challenge, often addressed within various settings. One such scheme is independent learning \cite{IL1}, where agents disregard the influence of others and interpret their actions as changes in the environment. In the context of partial observability, this approach actually hides the challenge posed by partial observability.   IL assumes that the joint optimal policy can be learnt  relying solely on local observations. However, this assumption may be valid only for basic problems, since treating other agents' behavior as environment may introduce instability to the underlying Markov Decision Process (MDP). \par
An alternative framework that is used in the literature is one that is based on a centralized training and decentralized execution (CTDE). Under this setting, agents need global information during the training phase but can rely solely on local observations for decision-making during execution. Notable CTDE algorithms include MADPPG \cite{MADPPG}, Q-mix \cite{QMIX}, and MAVEN \cite{MAVEN}. 
Despite being a more effective framework than IL, CTDE remains restrictive due to its requirement for joint training. \par  
Coordinated multi-agent RL (CoMARL)  proposed in \cite{Coordinated_MARL}, presents a more sophisticated setting due to considering the mutual influence among agents. However, even this approach does not entirely meet the demand for completely decentralized MARL with partial observability. It assumes existence of independent subgroups capable of determining the global state through centralized learning within each group.  In reality, situations may arise where observations from all agents are necessary to accurately identify the global state. In this case,  CoMARL basically becomes a centralized solution.
In our method, we delve into the scenario where agents remain {\em decentralized} both during training and decision-making phases while accounting for {\em partial} observability of the global state. \par 
In single-agent scenarios, reinforcement learning with partial observability is effectively addressed using Partially Observable Markov Decision Processes (POMDPs). However, its multi- agent extension, known as Dec-POMDP, has been proven to be  Non-deterministic Exponential-time (NEXP)
hard. Dec-POMDPs are inherently difficult to solve, with no formal solution available except for some approximate methods tailored to specific scenarios. Despite this, Dec-POMDPs remain widely used in MARL literature as they offer a manageable environment for modeling and analysis.
\subsection{State estimation in model-based learning}
In the approaches given in \cite{Similar_1, POMDP_IAC, POMDP_model},  the probability of global states, i.e., the belief vectors, are estimated using Bayesian rules. However, employing Bayesian updates in these contexts necessitates knowledge of transition models, posing a challenge for model-free reinforcement learning (RL) problems. Consequently, these approaches are more directly related to model-based RL or planning schemes, such as dynamic programming, where complete model information is available \cite{POMDP_book}. Additionally, the algorithm proposed in \cite{POMDP_model} focuses on policy evaluation rather than finding the optimal policy. In contrast, our proposed method is designed to be model-free, eliminating the need for a state transition model. It aims not only to evaluate given policies but also to discover the optimal policy in partially observable multi-agent environments.

\subsection{State estimation in model-free RL}

\subsubsection{Neural Network-based Approaches}
Some solutions for model-free MARL with partially observed states are based on neural network (NN) architectures \cite{Neural_POMDP, Actor_critic_POMDP, POMDP_RNN, Mean_field}. Algorithms like those in \cite{Actor_critic_POMDP,Neural_POMDP, Belief1, LSTM, NN_POMDP} are examples of NN-based solutions tailored for single-agent partially observable problems. Multi-agent extensions are discussed in \cite{POMDP_RNN, NN_model_free, NN_model_free2}. In most NN-based approaches, the primary concept involves using neural networks to directly map observations to observation-action values. However, replacing state-action values with observation-action values may not consistently yield effective results, as the success of the algorithm depends on how well the chosen model approximates the estimated values. NN-based solutions are typically considered brute-force methods and do not offer a suitable environment for analysis. NN-based algorithms mainly compete with each other in terms of   implementation characteristics such as time efficiency and complexity of NN architectures, relying largely on   experimental evidence. Our focus in this paper is on an analysis-based approaches. Given the difference in our objectives, we will therefore avoid comparing the proposed method with NN-based schemes.

\subsubsection{Non-neural-based Approach}
A non-NN-based model-free approach has been proposed in \cite{MA_POMDP}. This algorithm localizes all learning processes by decomposing cumulative reward estimates using zeroth-order policy optimization (ZOPO). A convergence proof is provided; however, it is only applicable for on-policy and Monte Carlo-based learning, such as the Reinforce algorithm. On-policy learning, as described earlier, is a more restrictive setting than off-policy learning, which is the focus of our work. Due to being one of the most recent model-free approaches,  we will compare our method with  ZOPO in Section \ref{sec:experiment}. 
\subsection{Consensus-based solutions}
The integration of belief diffusion to address Dec-POMDPs has been explored in prior studies. The authors in \cite{Consensus1, Similar_1} investigate the applicability of consensus protocols to Dec-POMDP settings, primarily focusing on online planning. However, this approach has not been thoroughly examined within the framework of reinforcement learning. These studies predominantly analyze the algorithm from a consensus protocol standpoint without specific convergence proofs for reinforcement learning scenarios. Furthermore, the methodologies proposed in these works are largely heuristic in nature. \par 
 The authors in \cite{Consensus2} incorporate belief diffusion into some basic reinforcement learning algorithms such as Q-learning. Our approach integrates belief diffusion into more advanced RL technique, specifically into the off-policy multi-agent actor-critic solutions, and provides detailed analysis and conditions  conditions for convergence.

\section{Model Setting}
\label{sec: Preliminary}
Our setting can be modelled through a Decentralized Partial Observable Markov decision process (Dec-POMDP) denoted by the tuple $(\mathcal{K}, \mathcal{A}_{\ell}, \{\mathcal{O}_{\ell}\}_{\ell=1}^{K}, \mathcal{S}, \{r_{\ell}\}_{\ell=1}^{K}, \mathcal{P},\{ \mathcal{L}_{\ell}\}^{K}_{\ell=1})$, where $\mathcal{K}\triangleq \{1,2,...,K\}$ is set of agents, $\{\mathcal{A}_{\ell}\}_{\ell=1}^{K}$ is set of  possible actions for agent $\ell$, $\mathcal{O}$ is a continuous set of observations, and $\mathcal{S}$ is a set of states. Moreover, $\mathcal{P}:\mathcal{S} \times \mathcal{A} \to \mathcal{S}$ is a transition model where the value $\mathcal{P}(s'|s,a)$ denotes the probability of transition to state $s' \in \mathcal{S}$ from state $s\in \mathcal{S}$ after taking joint action $a\triangleq \{a_{\ell}\in \mathcal{A}_{\ell}\}_{\ell=1}^K$. The likelihood function $\mathcal{L}_{\ell}(\xi|s)$ denotes the probability of observing $\xi \in \mathcal{O}_{\ell}$ when the true global state is $s\in \mathcal{S}$. 
Communication among agents is defined by some fixed graph $G$ where nodes represent the agents  and edges represent communicability of agents.   
\par 
Next, we introduce some important metrics for reinforcement learning. Let  $r_{\ell,n}\triangleq r_{\ell}(s_n, a_{\ell,n})$ denote the individual reward  achieved by agent $\ell$ at time $n$ if at global state $s_n \in \mathcal{S}$ it chooses action $a_{\ell,n} \in \mathcal{A}$. We denote the reward upper-bound by $R_{\max}$ and global reward at time $n$ by 
\begin{align}
\bar{r}_n= \frac{1}{K}\sum \limits_{\ell=1}^K r_{\ell, n}
\end{align}
Then, under some policy $\pi $, which drives the decisions by the agents,  the state-action value for state and action pair $(s,a)$ is defined as
\begin{align}
\label{eq:Q}
    Q^{\pi}(s, a)&\triangleq \mathbb{E}\left[\sum_{n=0}^{\infty} \gamma^n \bar{r}_n \mid s_{0}=s, a_0=a\right] 
\end{align}
 where  $\gamma$ is a discount factor. Similarly, a state value for the global state $s$ is defined as 
 \begin{align}
 \label{eq:V}
       V^{\pi}(s)&\triangleq \mathbb{E} \left[\sum_{n=0}^{\infty} \gamma^n \bar{r}_n \mid s_{0}=s\right]
 \end{align}
 Since the proposed algorithm is based on off-policy learning we introduce the notation for the target and behavioral policies. Let  $b_{\ell}:S\times A \to [0,1]$ denote individual policy of agent $\ell $. Individual target policies are approximated by local parameterized functions $\pi_{\ell}(\cdot;\theta_{\ell})$, where $\theta_{\ell} \in \mathbb{R}^{d}$. Policy functions $\pi_{\ell}$ are assumed to be continuously differentiable with respect to $\theta_{\ell}$. We assume the individual target policies to be independent. Therefore, we define the joint  target policy as the product of individual  target policies
 \begin{align}
   \pi=\prod_{\ell=1}^K \pi_{\ell} (\cdot;\theta_{\ell}): A \times S \rightarrow[0,1] 
\end{align}
Similarly, we define the joint behavioral policy as
\begin{equation}
b=\prod_{\ell=1}^K b_{\ell}: A \times S \rightarrow[0,1]
\end{equation}
Off-policy learning relies on corrections, which are implemented using the \textit{importance sampling ratio}, defined as:
\begin{align}
\label{eq:importance}
\rho_{n}\triangleq \frac{\pi_n}{b}=\prod\limits_{\ell=1}^K\limits   \frac{\pi_{
\ell,n}}{b_{\ell}}
\end{align}
As indicated by \eqref{eq:importance}, the importance sampling ratio necessitates joint policies. In the MAOPAC algorithm \cite{SUTTLE20201549}, the importance ratio is estimated in a decentralized manner at each actor-critic update by iteratively diffusing individual importance ratios until consensus:
\begin{align}
    \rho_{\ell, n}\triangleq \frac{\pi_{\ell,n}}{b_{\ell}}
\end{align}
We adopt this assumption in our work as well. In our analysis, we assume that consensus-based estimation is sufficiently accurate to ignore finite-time errors in the estimation of the joint $\rho_n$.

\section{Multi-agent Off-Policy Actor-Critic for  dec-POMDP }
\label{sec:algorithm}
In our study, we extend the multi-agent off-policy actor-critic (MAOPAC) algorithm described in \cite{SUTTLE20201549} by considering scenarios where the global state is not fully observed.  We refer to the proposed extended algorithm as MAOPAC-dec-POMDP. The MAOPAC-dec-POMDP algorithm is based on the repeated alternation of two main learning phases, as demonstrated in Figure \ref{fig:chart}. The first phase involves estimating the global state, while the second phase involves performing traditional MAOPAC with some adjustments. In the following subsections, we discuss these two phases in more detail. The complete listing of our method is presented in Algorithm \ref{Algorithm1}, which will be consistently referenced throughout our explanation. 
\subsection{Multi-agent off-policy actor-critic learning} 
\label{subsection1}
The MAOPAC algrotihm belongs to the family of policy gradient methods designed to concurrently learn the optimal policies (actors)  and corresponding state values (critics). In our study, we assume that each agent has its own estimate for state values, approximated using the following linear function:
\begin{align}
v_{\omega_{k,n}}\approx \mu_{k,n}^T\omega_{k,n}
\end{align}
where $\mu_{k,n}$ is a state feature vector  and $\omega_{k,n}$ is a parameter vector, also referred to as \textit{the critic parameter}, estimated by agent $k$ at time $n$. Target policies are approximated with some functions $\pi(\mu_{k,n};\theta_{k,n})$ parametrized by $\theta_{k,n}$:
\begin{align}
    \pi_{k,n}\approx \pi(\mu_{k,n};\theta_{k,n})
\end{align}
Therefore, the policy parameters $\theta_{k,n}$ will be referred to as \textit{ the actor parameters}. \par 
As given in steps \eqref{alg:omega1}-\eqref{alg:theta} of Algorithm \ref{Algorithm1}, the updates of $\theta_{k,n}$ and $\omega_{k,n}$ are similar to those of  traditional actor-critic algorithms with the exception of several correction factors: $\rho_{k,n}$, $e_{k,n}$ and $M^{\theta}_{k,n}$. The importance sampling ratio $\rho_{k,n}$ is essential for correcting the off-policy nature of learning.   Off-policy learning allows for a disparity between the agent's behavioral policy and the target policy, which offers practical advantages over on-policy approaches. Specifically,  when the behavioral policy  is related to the estimated target policy, as in on-policy learning, it may lead to insufficient exploration——a widely recognized challenge known as the exploration-exploitation trade-off. Nevertheless, off-policy learning causes instabilities.  The correction factor $\rho_{k,n}$ alone can address these instabilities when dealing with exact state and policy values. When employing function approximations, as in our study, additional corrections are required.  Therefore, a novel variant of the actor-critic paradigm, termed \textit{emphatic temporal difference learning} (ETD), has been proposed in \cite{ETD1, ETD2, SUTTLE20201549}. This variant ensures stability for linear approximation models and off-policy learning scenarios.  The correction variables $e_{k,n}$ and $M^{\theta}_{k,n}$ are  motivated by ETD. The steps \eqref{alg:F}-\eqref{alg:M_theta} involve intermediate calculations necessary for computing $e_{k,n}$ and $M^{\theta}_{k,n}$. 
 Parameters $\lambda, \zeta$ and  $\gamma $ are fixed and fall within the range $(0,1)$. 
In traditional MAOPAC, $\mu_{k,n}$ can be any vector with elements that fall within the range $[0,1]$. Therefore,  for our problem, we can treat 
$\mu_{k,n}$ as a belief vector of size $S$, where the $i$-th  element corresponds to the probability that the current state is 
$i$. In the case of full-state observation, as with MAOPAC, all elements of $\mu_{k,n}$ are zero except for $1$ at the location of the actual state.
\subsection{State estimation}
At each iteration of the critic and actor updates, agents perform state estimation using using the social learning strategy described by \eqref{belief_update}-\eqref{belief_update2}. Social learning is a form of group learning that identifies the most suitable hypothesis state from a set $\mathcal{S}$, best explaining the given observations $\xi_{k,i} \in \mathcal{O}_k$. In this work, we use the Adapt-Then-Combine social learning technique for state estimation \cite{Social_learning1,Social_learning2}. 
\begin{figure*}[h]
  \centerline{\includegraphics[width=19cm]{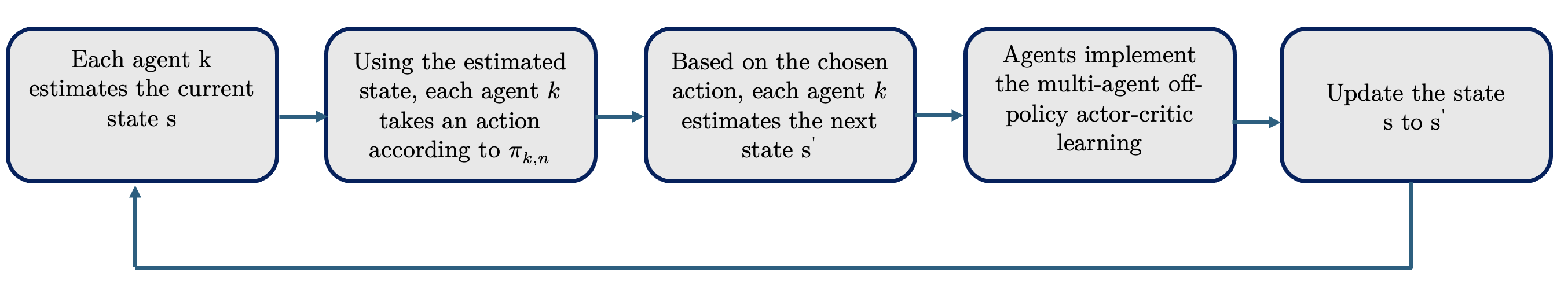}}
\caption{A block diagram illustrating the primary steps of the proposed algorithm.}
\label{fig:chart}
\end{figure*}
 At each iteration of MAOPAC learning, we implement an internal loop for learning the global state. At iteration $n$  of actor-critic updates, all agents receive a set of individual observations $\{\xi_{k,t}^n\}_{t=0}^T$, which are dependent on the current unknown global state of nature,  i.e., $s_n$. Using their individual likelihood functions $L_{k}(\xi_{k,n,t}|s)$, the $s$-th elements of the individual  belief vectors $\mu_{k,n}$  are updated as follows: 
 \begin{enumerate}
     \item Repeat \eqref{belief_update} and \eqref{belief_update2} for $t=0,1,...,T$, $\forall s\in \mathcal{S}$ and $\forall k \in \mathcal{K}$:
 \begin{align}
 \label{belief_update}
 &\text{Adapt: } \nonumber\\
&\psi_{k, t}(s)  =\frac{L_k\left({\xi}_{k,t}^n \mid s\right) {\widetilde{\mu}}_{k, t-1}(s)}{\sum_{s' \in \mathcal{S}} L_k\left(\xi_{k,t}^n \mid s^{\prime}\right) \widetilde{\mu}_{k, t-1}\left(s^{\prime}\right)} \\
 \label{belief_update2}
  & \text{Combine: } \nonumber\\
&\widetilde{\mu}_{k, t}(s) =\frac{\prod_{\ell \in \mathcal{N}_k}\left[{\psi}_{\ell, t}(s)\right]^{c_{\ell k}}}{\sum_{s^{\prime} \in \Theta} \prod_{\ell \in \mathcal{N}_k}\left[{\psi}_{\ell, t}(s)\right]^{c_{\ell k}}} 
\end{align}
\item  Assign:
 \begin{align}
 \label{eq:assign1}
     \mu_{k,n}=\widetilde{\mu}_{k,T}
 \end{align}
 \end{enumerate}
 where  $c_{\ell, k}$ are the entries of a combination matrix $C$ satisfying Assumption \ref{assumption:combination}. 
Actor-critic learning requires knowledge of the current and next states. Therefore, we also introduce the belief vector $\eta_{k,n}$  the $i$-th  element corresponds to the probability that the next state is 
$i$. Similar to updating  $\mu_{k,n}$, the next state belief vectors $\eta_{k,n}$ are updated as follows:     
\begin{enumerate}
     \item  Repeat \eqref{belief_update3} and \eqref{belief_update4} for $t=0,1,...,T$, $\forall s\in \mathcal{S}$ and $\forall k \in \mathcal{K}$ :\\
     Adapt: 
 \begin{align}
 \label{belief_update3}
\psi_{k, t}(s) & =\frac{L_k\left({\xi}_{k,t}^{n+1} \mid s\right) {\widetilde{\mu}}_{k, t-1}(s)}{\sum\limits_{s' \in \mathcal{S}} L_k\left(\xi_{k,t}^{n+1} \mid s^{\prime}\right) \widetilde{\eta}_{k, t-1}\left(s^{\prime}\right)} 
\end{align}
Combine:
\begin{align} 
   \label{belief_update4}
\widetilde{\eta}_{k, t}(s) & =\frac{\prod_{\ell \in \mathcal{N}_k}\left[{\psi}_{\ell, t}(s)\right]^{c_{\ell k}}}{\sum\limits_{s^{\prime} \in \Theta} \prod_{\ell \in \mathcal{N}_k}\left[{\psi}_{\ell, t}(s)\right]^{c_{\ell k}}} 
\end{align}
\item Assign:
 \begin{align}
 \label{eq:assign2}
     \eta_{k,n}=\widetilde{\eta}_{k,T}
 \end{align} 
 \end{enumerate}

Actually, agents are required to estimate both $\mu_{k,n}$ and $\eta_{k,n}$
  only during the initial iteration, i.e., for $n=0$. In subsequent iterations, the next state from the previous iteration is utilized as the current state in the new iteration, i.e.,
\begin{align}
    \mu_{k,n}=\eta_{k,n-1}
\end{align}

\begin{assumption}[\textbf{Combination matrix}]
\label{assumption:combination}
  The combination matrix $C$ assigns non-negative weights to neighboring agents and is  assumed to be doubly-stochastic, i.e., the entries on its  columns and rows sums to $1$: 
\begin{align}
\sum\limits_{\ell =1}^K c_{\ell,k}&=\sum\limits_{\ell \in \mathcal{N}_{k}}c_{\ell,k}=1 ; \\ \sum\limits_{\ell =1}^K c_{k,\ell}&=\sum\limits_{\ell \in \mathcal{N}_{k}}c_{k,\ell}=1
\end{align}  
where the notation ${\cal N}_k$ denotes the neighbors of agent $k$. \par \QEDA
\end{assumption}

As demonstrated in \cite{SL2}, under Assumptions \ref{assumption:combination}-\ref{assumption:likelihood}, repeated application of the updates in \eqref{belief_update} and \eqref{belief_update2} allows agents to almost surely learn the true current state $s_n$: 
\begin{align}
\mu_{k, n}\left(s_n\right) \stackrel{\text { a.s. }}{\longrightarrow} 1 \text {. }    
\end{align}
The convergence rate of the algorithm is discussed in \cite{SL_1}. Generally, the belief vectors converge  at an exponential rate, which depends on the second-largest-magnitude  eigenvalue of $C$ \cite{SL3}. Therefore, in Section \ref{sec:Theory}, we analyze the maximum allowable state estimation error to ensure that, by the end of the MARL run, the estimation of the policy parameter $\theta_{k,n}$ is $\varepsilon$-optimal. Using the maximum allowed state estimation error, we can adjust the run time of the inner loop, i.e., the number of iterations in the state estimation loop.
\begin{assumption}[\textbf{Strong connectivity}]
\label{assumption:graph}
    The underlying graph topology is assumed to be strongly connected, i.e, there exists a path with positive combination weights between any two agents in the network and, moreover, at least one agent $k_o$ has a nonzero self loop, $c_{k_o k_o}>0$. \par \QEDA
\end{assumption}
\begin{assumption}[\textbf{Likelihood function}]
\label{assumption:likelihood}
   For all agents $k\in \mathcal{K}$ and all states $s\in \mathcal{S}$, KL-divergence between the true model $f_k(\xi_{k})$ and the likelihood function  $L_k(\xi_{k}|s)$ is finite:
    \begin{align}
        \operatorname{D}_k\left(f\mid\mid L_k\right)\triangleq \mathbb{E}_{f_{k}}\log \frac{f_k(\xi)}{L_{k}(\xi\mid s)} <\infty
    \end{align} \QEDA
\end{assumption}
\begin{algorithm}
\SetAlgoLined
{\small
Initialize parameters: $\lambda\in (0,1)$, $\zeta\in (0,1)$, $\gamma\in (0,1)$, $\omega_{k,0}(s,a)$, $\rho_{k,0}$, $\widetilde{\mu}_{k,0}=\frac{1}{S}$,$\eta_{k,0}=\frac{1}{S}$ , $F_{k,0}=0$\;

\For{n=0,1,2....}{
\If{n=0}{
 Each agent $k$ receives observations $\{\boldsymbol{\xi}_{k, t}^{0}\}_{t=0}^{T}$ \;
Each agent $k$ estimates $\mu_{k,0}$ using \eqref{belief_update}-\eqref{eq:assign1}\; 

}
Each agent $k$ takes action $a_{k,n} \sim   b_{k} \left(a\| \mu_{k,n}\right)$\;
Each agent $k$ receives reward $r_{k,n}$ \; 
Each agent $k$ receives observations $\{\xi_{k,t}^{n+1}\}_{t=0}^T$\; 
Each agent $k$ estimates $\eta_{k,n}$ using \eqref{belief_update3}-\eqref{eq:assign2}\;

Each agent $k$ estimates $\rho_{k,n}$ using \eqref{eq:p_kn}-\eqref{eq:rho_kn}\;
\For{each agent $k$}{
{
\begin{align}
& F_{k,n}=1+\gamma \rho_{k,n-1} F_{k,n-1} \label{alg:F} \\
& M_{k,n}=\lambda+(1-\lambda) F_{k,n}  \label{alg:M}\\
& e_{k,n}=\gamma \lambda e_{k,n-1}+M_{k,n} \mu_{k,n} \label{alg:e}\\
& M_{k,n}^\theta=1+\zeta \gamma \rho_{k,n-1} F_{k,n-1} \label{alg:M_theta}\\
& \delta_{k,n}=r_{k,n}+\gamma {\omega_{k,n}^T}\eta_{k,n}-{\omega_{k,n}^T}\mu_{k,n} \\
&\Psi_{k,n}=\frac{\nabla_{\theta}  \pi\left(a_{k,n} \mid \mu_{k,n}\right)}{\pi\left(a_{k,n} \mid \mu_{k,n}\right)}\\
&\widetilde{\omega}_{k,n}=\omega_{k,n}+\beta_{n} \rho_{k,n} \delta_{k,n} e_{k,n} \label{alg:omega1}\\
&\theta_{k,n+1}=\theta_{k,n}
+\beta_{n} \rho_{k,n} M_{k,n}^\theta  \delta_{k,n} \Psi_{k,n} \label{alg:theta}
\end{align}}
}
\For{each agent $k$ }{
\begin{equation}
\omega_{k,n+1}=\sum_{\ell \in \mathcal{N}_k} c_{\ell, k} \widetilde{\omega}_{\ell, n} \label{alg:omega2}
\end{equation}}
Assign: $\mu_{k,n}=\eta_{k,n}$\; 
Reset:  $\widetilde{\eta}_{k,0}(s)=\frac{1}{S}$  \;
}
\caption{MAOPAC-dec-POMDP}
 \label{Algorithm1}
 }
\end{algorithm}
\subsection{Importance sampling ratio} The importance sampling ratio is crucial for correcting off-policy algorithms. Similar to state estimation, the importance sampling ratio also needs to be learned in a distributed manner because each agent, due to privacy concerns, is only aware of its own behavioral and estimated target policies. Specifically, we introduce the variable $p_{k,n}$ to denote log importance ratios:
\begin{equation}
\label{eq:p_kn}
    p_{k,n}\triangleq\log \frac{\pi_{k,n} \left(a| \mu_{k,n}\right)}{b_{k} \left(a| \mu_{k,n}\right)}
\end{equation}
Next, for $t=0, 1,...T_{\rho}$, agents diffuse  $p_{k,n}$ as follows:
  \begin{equation}
\begin{gathered}
\label{eq:tildep_kn}
\widetilde{p}_{k, t} \propto \sum_{\ell \in \mathcal{N}_k} c_{\ell k}\widetilde{p}_{\ell, t}
\end{gathered}
\end{equation}
where $\widetilde{p}_{k,0}=p_{k,n}$. 
After $T_{\rho}-$iterations, the improved   importance  sampling ratios $\rho_{k,n}$ are retrieved as:
\begin{align}
\label{eq:rho_kn}
 \rho_{k,n}=\exp \left(N_k \widetilde{p}_{k,T_{\rho}}\right) 
\end{align}
For simplicity of analysis, at this stage of the study, we assume that $T_\rho$ is large enough to approximate the individual estimate $\rho_{k,n}$ with the global importance sampling ratio, i.e., 
\begin{align}
\label{eq:rho_main}
\rho_{k,n} \approx \prod\limits_{k' \in K} \frac{ \pi(a_{k',n} | \mu_{k',n}; \theta_{k',n})}{b_{k'}}
\end{align}
To align with the original algorithm \cite{SUTTLE20201549}, we introduce Assumption \ref{assumption:policies} related to behavior policies, which leads to the following constraint:
\begin{align}
\label{eq:rho_constraint}
\rho_n \leq \frac{1}{b_{ {\epsilon}}}
\end{align}
\begin{assumption}[\textbf{Behavioral policies}]
\label{assumption:policies}
The joint behavioral policy is time-invariant and nonzero, i.e., $b_{n} \geq  b_{\epsilon} > 0$. Additionally, for analysis, we require that $ b_{\epsilon} > \gamma$. \par \QEDA
\end{assumption}
The discount factor $\gamma$ is commonly chosen in the range of $0.8-0.9$. Therefore, Assumption \ref{assumption:policies} requires that the probability of the chosen action at time $n$ considerably dominates over other possible actions. A good example is a deterministic policy where one action whose probability (namely $1$) considerably dominates over the probabilities of other actions (zeros). The learning rate at time $n$, denoted by $\beta_{n}$, follows Assumption \ref{assumption:step}. 
 \begin{assumption}[\textbf{Learning rate}]
\label{assumption:step}
The learning step size $\beta_{n}$ satisfies
\begin{align}
    \sum \limits_{n=0}^\infty \beta_{n}=\infty, \quad \quad \sum \limits_{n=0}^\infty \beta_{n}^2<\infty
\end{align}
\QEDA
\end{assumption}



\section{Theoretical guarantees}
\label{sec:Theory}
For analysis we compare the performance of MAOPAC when the global state is fully observed against when it is only partially observed. In essence, we compare the original algorithm MAOPAC and the proposed MAOPAC-dec-POMDP. Let $\widehat{\mu}_{k,n}$, $\widehat{\rho}_{k,n}$, $\widehat{\omega}_{k,n}$, $\widehat{\theta}_{k,n}$, $\widehat{F}_{k,n}$, $\widehat{M}_{k,n}$, $\widehat{e}_{k,n}$, $\widehat{\delta}_{k,n}$, and $\widehat{M}^{\theta}_{k,n}$ denote the full-observation counterparts of $\mu_{k,n}$, $\rho_{k,n}$, $\omega_{k,n}$, ${\theta}_{k,n}$, ${F}_{k,n}$, ${M}_{k,n}$, ${e}_{k,n}$, ${\delta}_{k,n}$, and ${M}^{\theta}_{k,n}$, respectively. These variables undergo the same update processes as their counterparts, with the distinction that they are privileged with  knowledge of the true global state. For comparison between these two sets we introduce the error variables
\begin{align}
    &\Delta \mu_{k,n}\triangleq \widehat{\mu}_{k,n}-\mu_{k,n}, \ \Delta \rho_{k,n}\triangleq \widehat{\rho}_{k,n}-\rho_{k,n} ,\\ &\Delta\omega_{k,n} \triangleq \widehat{\omega}_{k,n}-\omega_{k,n}, \ \Delta \theta_{k,n}\triangleq \widehat{\theta}_{k,n}-\theta_{k,n},\\
     &\Delta F_{k,n}\triangleq \widehat{F}_{k,n}-F_{k,n},  \ \Delta \widehat{\delta}_{k,n}\triangleq \widehat{\delta}_{k,n}-\delta_{k,n}
\\ &\Delta e_{k,n}\triangleq \widehat{e}_{k,n}-e_{k,n}, \   \Delta M_{k,n}\triangleq \widehat{M}_{k,n}-M_{k,n}\\
&\Delta M_{k,n}^{\theta}\triangleq \widehat{M}_{k,n}^{\theta}-M_{k,n}^{\theta}
\end{align}
Additionally, in the analysis of the proposed method, we must employ the following upper bounds:
\begin{align}
    &|M^{\theta}_{k,n}|\leq  \frac{1-(1-\zeta)\gamma/ b_{\epsilon}}{1-\gamma/ b_{\epsilon}} \triangleq  B_M^{\theta}\\
    &\|\omega_{k,n}\|\leq  \sum\limits_{i=0}^{n-1}  \frac{\beta_{i}\Omega^{n-i}R_{\max}  B_e \|\omega_{0, \max}\|}{ b_{\epsilon}} \triangleq B_n^{\omega}\\
    &|\delta_{k,n}|\leq R_{\max}+(1+\gamma) B^\omega_{n}\triangleq B_{n}^{\delta}\\
    &\|e_{k,n}\|\leq \frac{1}{1-\lambda\gamma}\left(\lambda+\frac{1-\lambda}{1-\gamma/ b_{\epsilon}}\right)\triangleq  B_e
\end{align}
These bounds are derived in the appendix (Lemma \ref{lemma:M_bound} - \ref{lemma:omega_bound2}).\par  
The primary objective of the MAOPAC algorithm is to find the policy parameter $\widetilde{\theta}_{k,n}$ that maximizes the average cumulative reward across the network. The convergence of MAOPAC to the optimal solution has been established in \cite{SUTTLE20201549}. In our analysis, we determine conditions under which the actor parameter under partial observability,   $\theta_{k,n}$, can get $\varepsilon$-close to the optimal result of MAOPAC. This finding is formally stated in Theorem \ref{theorem2}, whose proof appears in the appendix.
\begin{theorem}[\textbf{$\boldsymbol{\epsilon}$-optimality}]
\label{theorem2}
Let $\|\cdot\|$ denote the Euclidean norm ($2-$norm) of a vector. Then,
    under Assumptions \ref{assumption:combination}-\ref{assumption:step}, for all agents $k$, $\Delta \theta_{k,n}$  is $\varepsilon$-bounded at time $n$ if   $\forall j \leq n$: 
\begin{align}
\label{eq:theorem21}
    \|\Delta \mu_{\ell ,j})\|\leq \min (\widetilde{B}_1,\widetilde{B}_2 ) \text{  and }
\end{align}
 \begin{align}
 \label{eq:theorem22}
     \|\Delta \rho_{\ell,j}\|  \leq \min  (\widetilde{D}_1, \widetilde{D}_2
     ,\widetilde{D}_3)
 \end{align}
where
\begin{align} 
    \widetilde{B}_1 &\triangleq  \frac{\varepsilon  b_{\epsilon}(1+\gamma)^{-1}B_e^{-1} }{\Phi_n \beta_j B^{\omega}_{j} \Omega^{n-j} }\label{eq:B_1_main}, \\      \widetilde{B}_2&\triangleq \frac{\varepsilon\beta_0^{-1} I_1^{-1} }{2\Phi_n |M_{k,j}|B_{n}^{\delta}\Omega^{n-I_1}(\gamma \lambda)^{I_1-j}}\label{eq:B_2_main}\\
\widetilde{D}_1 &\triangleq \frac{\varepsilon  b_{\epsilon}B_e^{-1}}{\Phi_n \beta_j B_j^{\delta}\Omega^{n-j}} \label{eq:D_1_main}  \\
 \widetilde{D}_2 &\triangleq \left(\frac{ b_{\epsilon}}{\gamma}\right)^{I_2-j}\frac{\varepsilon I_2^{-2} (1-\lambda)^{-1}\beta_0^{-1}}{2\Phi_nF_{k,j} B_{ n }^{\delta}\Omega^{n-I_2}} \label{eq:D_2_main}\\
 \widetilde{D}_3&\triangleq\left(\frac{ b_{\epsilon}}{\gamma}\right)^{I_3-j} \frac{3 \varepsilon b_{\epsilon}I_3^{-2}\zeta^{-1}}{8\beta_n  \pi^2   B^{\delta}_{n} F_{k,j}} \label{eq:D_3_main} \\
  I_1&\triangleq \left[\ln\frac{\Omega}{\gamma \lambda}\right]^{-1}, \quad
   I_2\triangleq 2 \left[\ln \frac{ b_{\epsilon}\Omega}{\gamma}\right]^{-1}, \\ 
   I_3&\triangleq2 \left[\ln\frac{ b_{\epsilon}}{\gamma}\right]^{-1}, \quad \Omega \triangleq 1+\frac{\beta_0(1+\gamma)B_e}{b_{\epsilon}}\\ 
   \Phi_n &\triangleq 4\beta_{0}(1+\gamma)\pi^2 B^{\theta}_{M} n^3,
\end{align} \QEDA
\end{theorem}
 \begin{figure*}[h]
\begin{minipage}[b]{0.5\linewidth}
  \centering
  \centerline{\includegraphics[width=7.5cm]{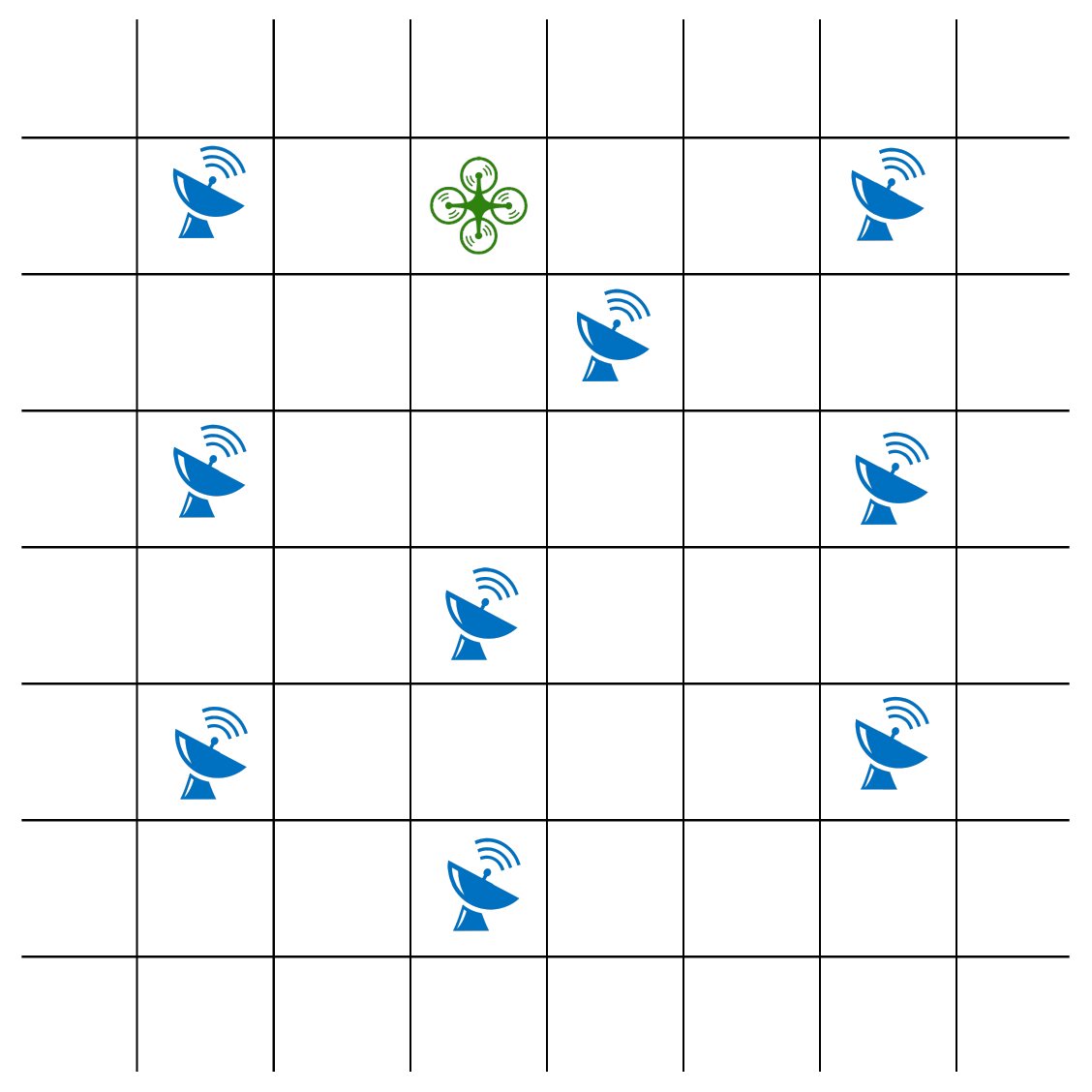}}
  \centerline{(a) Initial state  }\medskip
\end{minipage}
\begin{minipage}[b]{0.5\linewidth}
  \centering
  \centerline{\includegraphics[width=7.5cm]{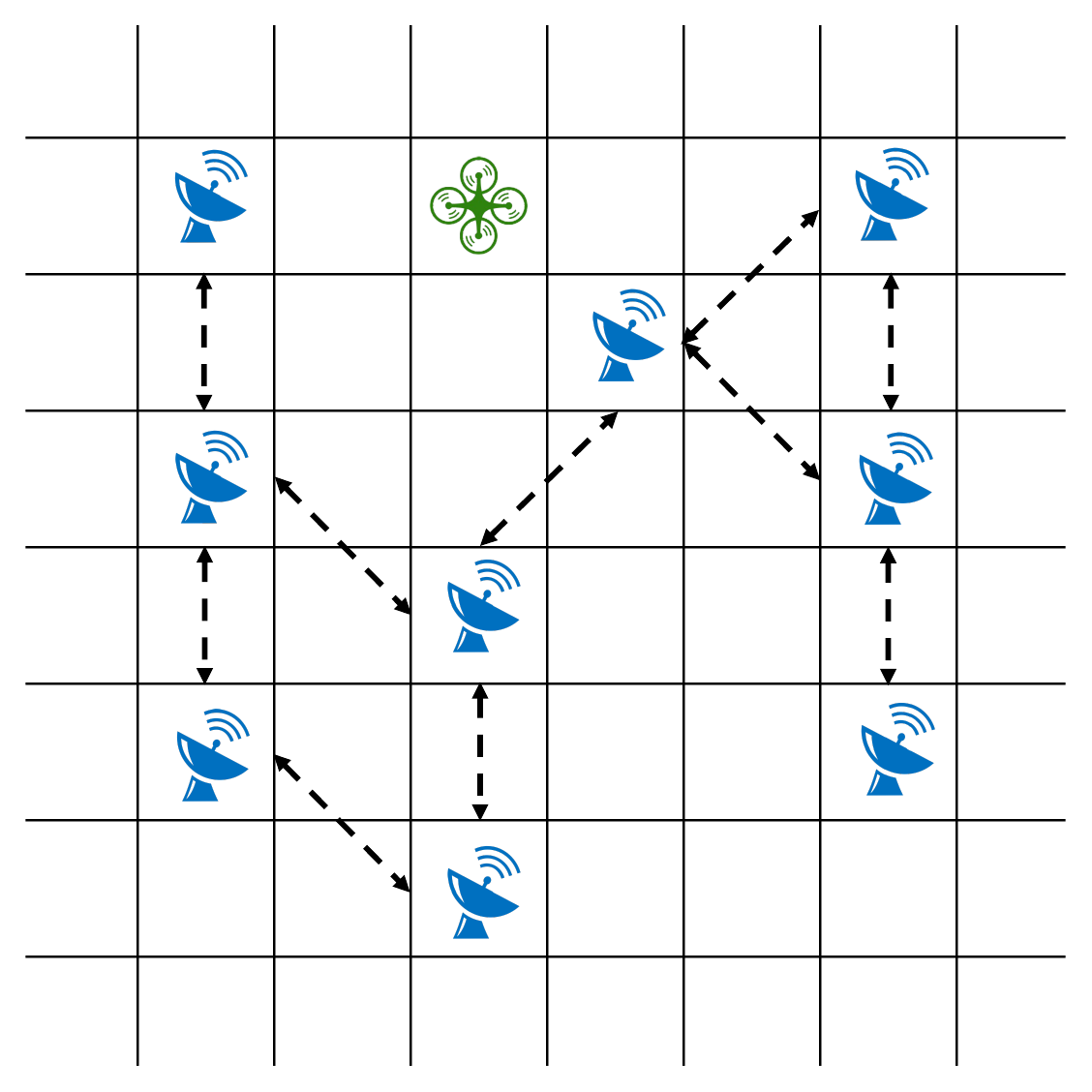}}
\centerline{(b) Parameter exchanges  }\medskip
\end{minipage}
\begin{minipage}[b]{0.5\linewidth}
  \centering
  \centerline{\includegraphics[width=7.5cm]{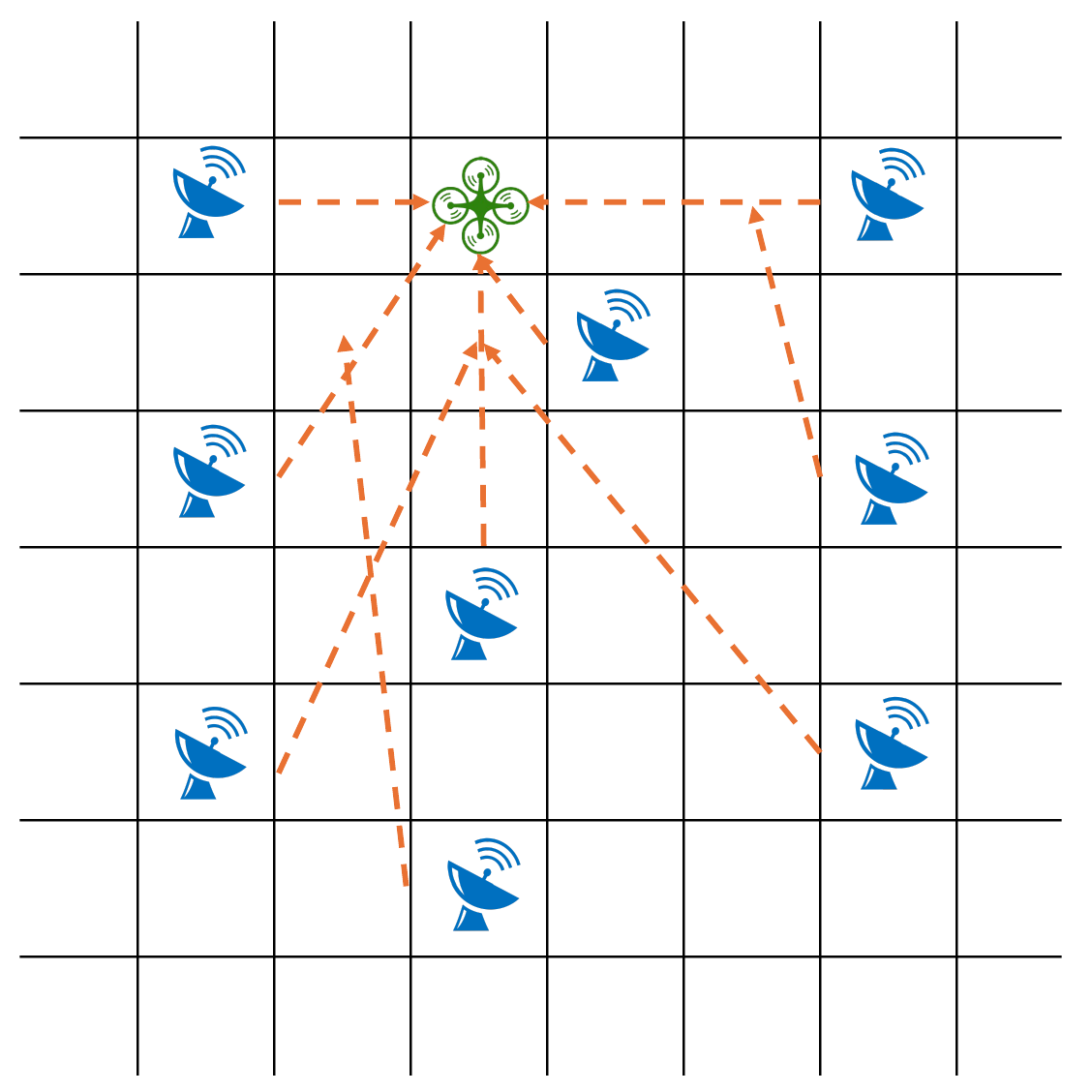}}
  \centerline{(c) Action selection }\medskip
  
\end{minipage}
\begin{minipage}[b]{0.5\linewidth}
  \centering
  \centerline{\includegraphics[width=7.5cm]{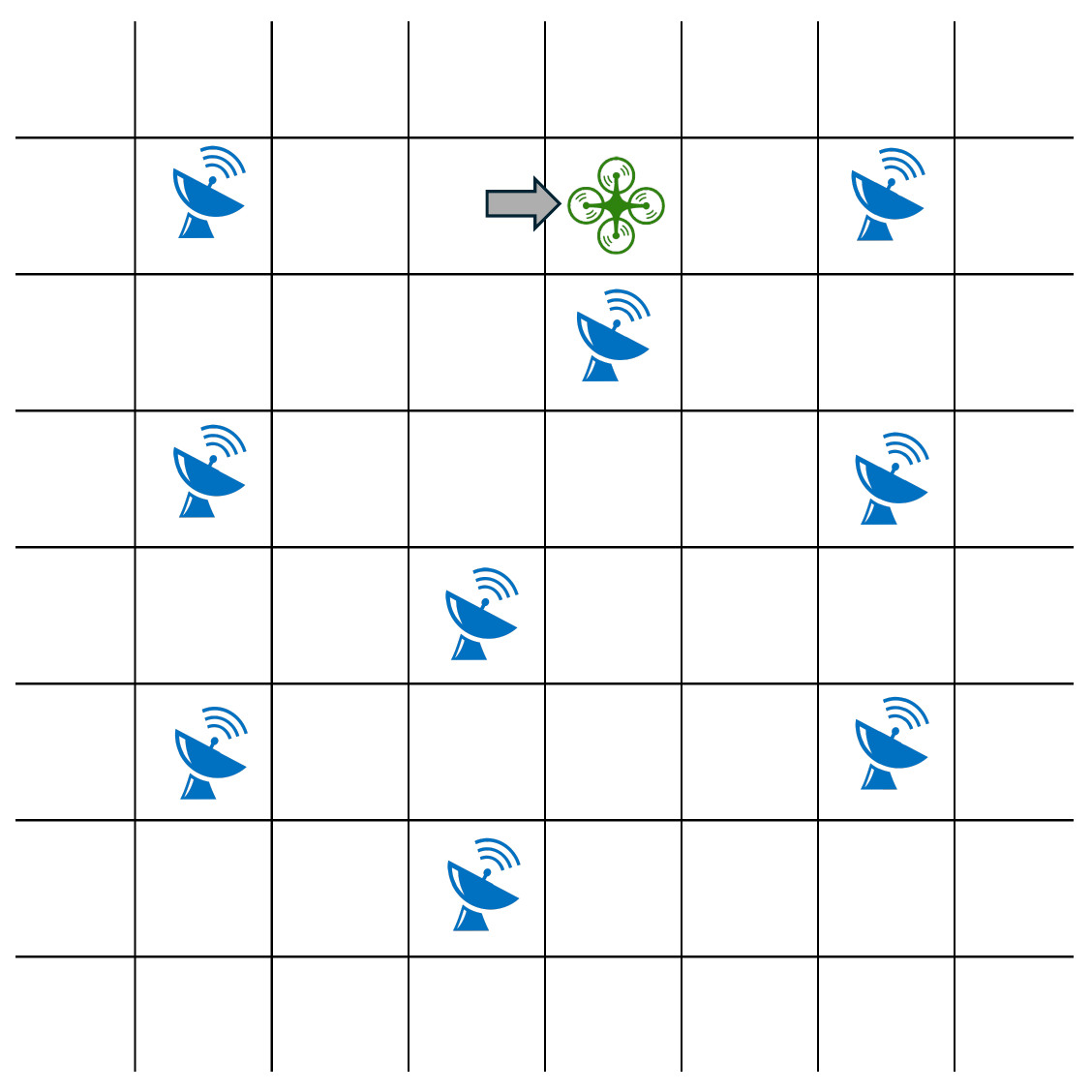}}
  \centerline{(d) State transition }\medskip
\end{minipage}
\caption{Illustration of the agents/target framework: (a) depicts the initial state, corresponding to the position of a target (green) and the fixed positions of agents (blue); (b) shows a phase where agents exchange parameters according to the communication graph: black arrows demonstrate communication links; (c) illustrates a phase where agents, based on their individual policies, choose an action, i.e., select the possible location of the target (cells indicated by the orange arrows);  (d) demonstrates the transition of the target to another state (cell) as a result of the agents' actions.}
\label{fig:experiment}
\end{figure*}
Note that over infinite time, $n \to \infty$, the bounds in \eqref{eq:theorem21} and \eqref{eq:theorem22} converge to zero, which is a reasonable outcome. State estimation occurs at every iteration of the actor-critic updates. As a result, each new state estimation error $\Delta\mu_{k,n}$ contributes to the overall error of the actor parameter $\Delta \theta_{k,n}$. Therefore, the convergence of the actor error $\Delta \theta_{k,n}$ necessitates the convergence of $\Delta\mu_{k,n}$ to zero as $n$ approaches infinity.
Next, note that the bounds in \eqref{eq:theorem21} and \eqref{eq:theorem22} can be computed in real time, highlighting the practical importance of these bounds. Specifically, an agent undergoing an actor-critic update at time $j$ needs to determine the maximum allowable state estimation error and use it to adjust the runtime of the inner state-estimation loop. For this purpose, they can refer to the bounds in \eqref{eq:theorem21} and \eqref{eq:theorem22}, which require only the current time $j$ variables and constants, without the need to retain the history of previous updates.\par
In addition, we derive Corollary \ref{corollary1} which extends  Theorem \ref{theorem2} for a specific target policy model. We assume the target  and behavioral policies follow Boltzmann distributions parametrized by vectors $\{\theta_{c,n}\}_{\forall c\in \mathcal{A}}$ and $\{\bar{\theta}_{c,n}\}_{\forall c\in \mathcal{A}}$, respectively:  
\begin{align}
\label{eq:Boltzman}
\pi_{k,n}(a\|\mu_{k,n})&=\frac{\exp[\mu_{k,n}^T \theta_{a,n}]}{\sum\limits_{c \in \mathcal{A}}\exp[\mu_{k,n}^T \theta_{c,n}]}  \\ \quad b_{k,n}(a\|\mu_{k,n})&=\frac{\exp[\mu_{k,n}^T \bar{\theta}_{a,n}]}{\sum\limits_{c \in \mathcal{A}}\exp[\mu_{k,n}^T \bar{\theta}_{c,n}]}
\end{align}
\begin{cor} [\textbf{Boltzman distributions}]
\label{corollary1}
Let the behavioral and target policies be defined according to \eqref{eq:Boltzman}. 
Then, for all agents $k\in \mathcal{K}$, the actor parameter estimation error $\Delta \theta_{k,n}$ is $\varepsilon$-bounded if the state estimation error $\Delta \mu_{k,n}$ is bounded as 
    \begin{align}
    \|\Delta\mu_{k,n}^T\| \leq \min(\widetilde{B}_1,\widetilde{B}_2,\widetilde{B}_3)
\end{align}
where $\widetilde{B}_1$ 
 and $\widetilde{B}_2$ are defined in \eqref{eq:B_1_main} and \eqref{eq:B_2_main}, respectively, and 
\begin{align}
&\widetilde{B}_3\triangleq  \frac{1}{\|\Theta\|}\ln \left[ \frac{\min  (\widetilde{D}_1, \widetilde{D}_2
     ,\widetilde{D}_3)+\rho_{k,n}(a)}{\exp[\mu_{k,n}^T \Theta]}\right]\\
     &\Theta \triangleq \theta_a-\bar{\theta}_a+\bar{\theta}_{\max}-\theta_{\min}\\
     &\theta_{\min}\triangleq \min \limits_{c\in\mathcal{A}}\theta_{c},  \quad \bar{\theta}_{\max}\triangleq \max \limits_{c\in\mathcal{A}}\bar{\theta}_{c}
\end{align} \QEDA
\end{cor}
The Boltzmann distribution is a common choice for modeling policies in reinforcement learning. In fact, the impact of the particular policy model is confined to the concluding steps of the proof for Theorem \ref{theorem2}.  Consequently, it is reasonable to conjecture that replicating such an analysis for alternative distributions would entail minimal additional effort.\par 
It is important to highlight that the behavioral policy in the proposed scheme is assumed to be time-invariant. While this assumption holds in many cases, it could pose challenges for certain specially designed behavioral policies. For instance, in \cite{Zhaikhan}, the behavioral policy is crafted to ensure sufficient exploration of the state space, which may require time-varying characteristics for optimal performance. Therefore, future study may include a version with time-variant behavioral policies.
\begin{figure*}[]
\begin{minipage}[b]{0.5\linewidth}
  \centering
  \centerline{\includegraphics[width=10cm]{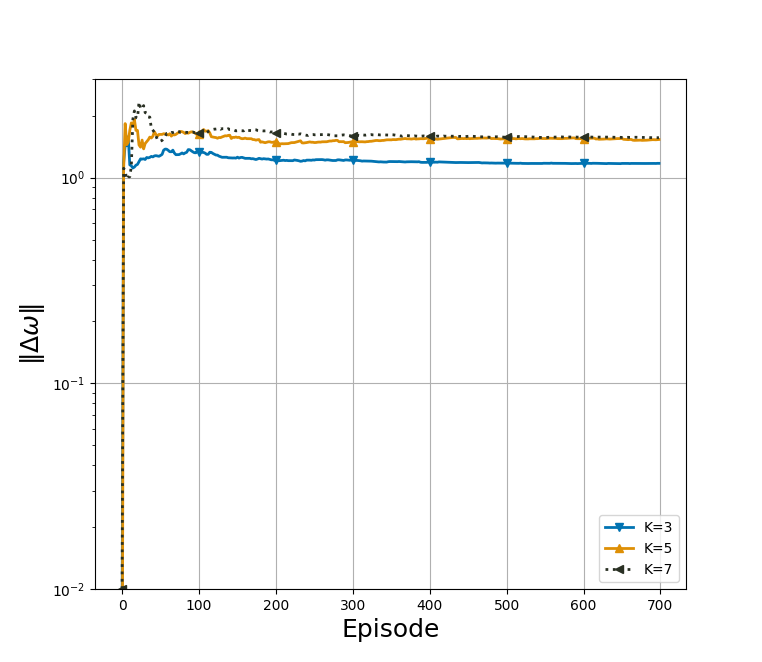}}
  \centerline{(a) Critic error  }\medskip
\end{minipage}
\begin{minipage}[b]{0.5\linewidth}
  \centering
  \centerline{\includegraphics[width=10cm]{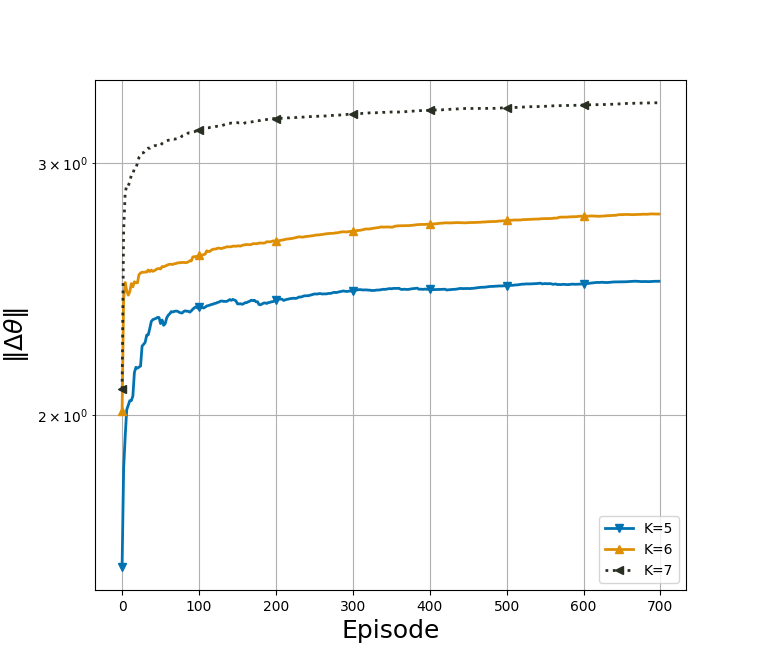}}
\centerline{(b) Actor error  }\medskip
\end{minipage}
\begin{minipage}[b]{0.5\linewidth}
  \centering
  \centerline{\includegraphics[width=10cm]{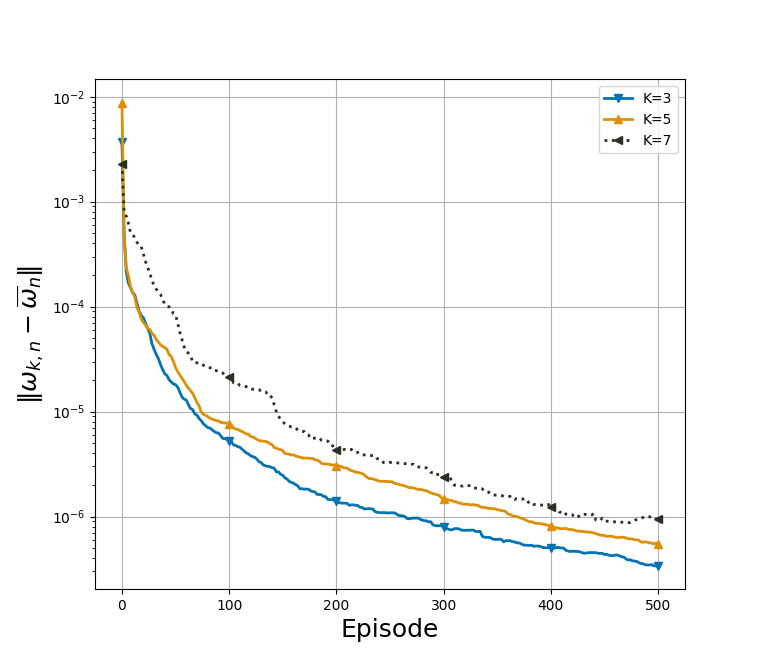}}
  \centerline{(c) Critic agreement }\medskip
  
\end{minipage}
\begin{minipage}[b]{0.5\linewidth}
  \centering
  \centerline{\includegraphics[width=10cm]{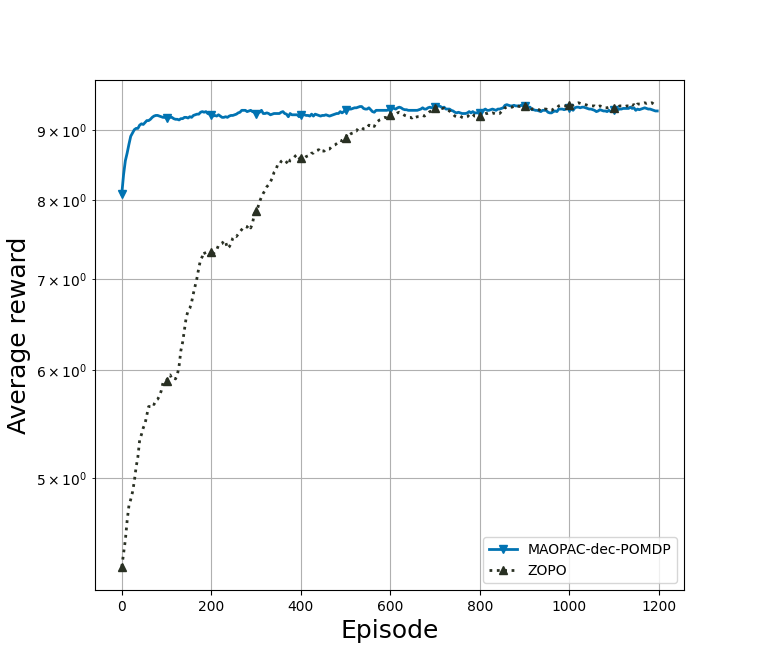}}
  \centerline{(d) MAOPAC-dec-POMDP vs ZOPO }\medskip
\end{minipage}
\caption{Comparison between MAOPAC and the proposed MAOPAC-dec-POMDP: (a) shows the difference in critic values computed by MAOPAC and MAOPAC-dec-POMDP for different numbers of agents: $K=3,K=5$ and $K=7$; (b) shows the difference in actor values computed by MAOPAC and MAOPAC-dec-POMDP for different numbers of agents: $K=5,K=7$ and $K=9$ (c) shows the network agreement for different numbers of agents: $K=3,K=5$ and $K=7$; (d) compares the proposed MAOPAC-dec-POMDP and ZOPO in terms of cumulative average reward}
\label{fig:res1}
\end{figure*}
\section{Experimental results and Discussions}
\label{sec:experiment}
For experiments we use the grid-based scenario from  \cite{Policy_evaluation}. It involves $K$ agents (representing radars) with fixed location  and one object with changing location over $h\times h$ grid (see Figure \ref{fig:experiment}).  The states of the underlying POMDP  are the cells of the grid. The objective for all agents is to correctly detect the location of the moving object (i.e., the cell where it is present). Hence, actions of the POMDP are also cells of the grid. 
The actions by  agents  correspond to which cell in the grid to hit.  Agents are rewarded based on the accuracy  of their selections relative  the actual location of the object. Furthermore, the object tries to move away from the most recent hit locations. In this way, the  transition between states is dependent on the actions chosen by agents. \par 
Figure \ref{fig:res1}(a) and Figure \ref{fig:res1}(b) illustrate the absolute difference in estimating the critic and actor parameters between two scenarios: one where the global state is only partially observed (MAOPAC-dec-POMDP), and the other where it is fully observed (MAOPAC). As depicted in these figures, both critic and actor values for MAOPAC-POMDP closely align with that of MAOPAC.   Convergence  of MAOPAC to the optimal policy under full observability has been proven in \cite{SUTTLE20201549}. Hence, it follows that MAOPAC-POMDP can attain an $\varepsilon$-optimal solution. \par 
While each agent possesses its own actor parameter $\theta_{k,n}$, the introduction of the critic parameter $\omega_{k,n}$ serves to model global state values. Consequently, it is imperative that the critic variables $\{\omega_{k,n}\}_{k=1}^{K}$ across all agents agree with each other. This requirement is substantiated by Figure \ref{fig:res1}(c), which depicts the difference (averaged across all agents) between individual critic parameters $\omega_{k,n}$ and the network average:
\begin{align}
\bar{\omega} \triangleq \frac{1}{K}\sum_{k=1}^{K}\omega_{k,n}.
\end{align}
\subsection{Comparision between MAOPAC-dec-POMDP and ZOPO}   
ZOPO (Zeroth-Order Policy Optimization) is an extension of Monte-Carlo-based policy gradient approaches such as the REINFORCE algorithm. ZOPO does not require knowledge of the actual gradient of the objective function (expected reward). Instead, ZOPO approximates it using samples of the objective function under perturbed parameters.\par 
ZOPO is inherently simple to implement and can be useful when the gradient of a function is not available. However, it tends to exhibit slow convergence and high noise levels, as confirmed in Figure \ref{fig:res1}(d). As previously mentioned, ZOPO relies on samples of the expected reward to compute estimates, necessitating multi-step sampling for each new update of the policy parameters. \par 
In contrast, the proposed MAOPAC-dec-POMDP (Multi-Agent Off-Policy Actor-Critic with Partial Observability) does not require such extensive sampling but instead relies on an internal loop for state estimation. Therefore,  both algorithms have similar levels of complexity.\par 
Moreover, MAOPAC-dec-POMDP can operate off-policy, which is generally more advantageous than on-policy solutions. Off-policy methods are more practical, easier to implement, and broaden the range of potential applications. 
\section{Conclusion}
\label{sec:Conclusion}
This paper proposes a multi-agent off-policy actor-critic algorithm for partially observable environments. The key innovation lies in estimating the global state through social learning to guarantee $ b_{\epsilon}-$boundedness of estimation error. The performance of the resulting algorithm is illustrated by comparing against state-of-the-art solutions. Currently, the algorithm's implementation necessitates a two-time scale learning approach. Future extensions would focus on streamlining the learning process into a single time-scale.  Additionally, extending the algorithm to accommodate time-variant behavioral policies opens doors to diverse applications, notably in exploration problems.
\section*{Acknowledgment}
\begin{appendices}
\section{Proof of Theorem 1}
\subsection{Bounds of intermediate variables}
\begin{lemma}
\label{lemma:M_bound}
    $M_{k,n}$ can be upper-bounded by $B_M$ defined as 
    \begin{align}
        B_M\triangleq \lambda+\frac{(1-\lambda)}{1-\gamma/ b_{\epsilon}}
    \end{align} \QEDA
\end{lemma}
\begin{proof}
From  \eqref{alg:F} we get 
\begin{align}
F_{k,n}=\sum\limits_{i=0}^n \prod_{j=i+1}^n \gamma\rho_i\stackrel{\text { (a) }}{\leq}\sum\limits_{i=0}^n \prod_{j=i+1}^n \frac{\gamma}{ b_{\epsilon}}=\frac{1}{1-\frac{\gamma}{ b_{\epsilon}}} 
\end{align}
where in (a) we assumed $\rho_n \leq \frac{1}{ b_{\epsilon}}$ and (b) is valid because $\frac{\gamma}{ b_{\epsilon}}\leq 1$ by Assumption \ref{assumption:policies}. 
Therefore, using the definition of $M_{k,n}$ in \eqref{alg:M}, we can get the following bound 
\begin{align}
M_{k,n}\leq \lambda+\frac{(1-\lambda)}{1-\gamma/ b_{\epsilon}}\triangleq B_M
\end{align}
\end{proof}
\begin{lemma}
\label{lemma:e_bound}
$\|\widehat{e}_{k,n}\|$ and $\|e_{k,n}\|$  can be upper-bounded by $B_e$ defined as
\begin{align}
    B_e\triangleq \frac{1}{1-\lambda\gamma}\left(\lambda+\frac{1-\lambda}{1-\gamma/ b_{\epsilon}}\right)
\end{align} \QEDA
\end{lemma}
\begin{proof}
  From the update rule for $e_{k,n}$ 
\begin{align}
e_{k,n}=\sum\limits_{i=0}^{n-1}(\gamma \lambda)^{n-1-i} M_{i} \mu_{k,i}
\end{align}
Therefore, $\|e_{k,n}\|$ can be bounded as 
\begin{align}
\|e_{k,n}\|&\leq \sum\limits_{i=0}^{n-1}(\gamma \lambda)^{n-1-i} |M_{k,i}| 
\|\mu_{k,i}\|\nonumber\\
&\stackrel{\text {(a)}}{\leq }\sum\limits_{i=0}^{n-1}(\gamma \lambda)^{n-1-i} B_M\nonumber\leq \sum\limits_{i=0}^{\infty}(\gamma \lambda)^{n-1-i} B_M \nonumber\\
&\stackrel{\text {(b)}}{= } \frac{1}{1-\lambda\gamma}\left(\lambda+\frac{1-\lambda}{1-\gamma/ b_{\epsilon}}\right)\triangleq B_e  
\end{align}
where  (a) is valid due to $\|\mu_{k,n}\|\leq 1 $ and  Lemma \ref{lemma:M_bound} and (b) is valid if we choose $\lambda\gamma\leq 1$. Similarly, we can show that 
\begin{align}
\|\widehat{e}_{k,n}\|\leq B_e
\end{align}
\end{proof}
\begin{lemma}
\label{lemma:M_theta_bound}
    $M^{\theta}_{k,i}$ can be upper-bounded by $B_M^{\theta}$ defined as
    \begin{align}
        B_M^{\theta} \triangleq \frac{1-(1-\zeta)\gamma/ b_{\epsilon}}{1-\gamma/ b_{\epsilon}} 
    \end{align} \QEDA
\end{lemma}
\begin{proof}
     We start from the update rule  of  $M^{\theta}_{k,i}$
\begin{align}
\label{eq:M1}
 M_{k,n}^\theta=1+\lambda_\theta \gamma \rho_{k,n-1} F_{k,n-1} 
\end{align}
$ F_{k,n}$ can be upper-bounded by
\begin{align}
\label{eq:F1}
     F_{k,n}&=1+\gamma \rho_{n-1} F_{k,n-1} = \sum \limits_{i=0}^{n-1} \prod_{j=i+1}^{n-1}(\gamma \rho_{j})F_{k,0}\nonumber\\
     &\stackrel{\text { (a) }}{\leq} \sum \limits_{i=0}^{n-1} \left(\frac{\gamma}{ b_{\epsilon}}\right)^{n-i-1} 
     \leq\sum \limits_{i=0}^{\infty} \left(\frac{\gamma}{ b_{\epsilon}}\right)^{i}\stackrel{\text { (b) }}{=}\frac{1}{1-\gamma/ b_{\epsilon}}
\end{align}
where in (a) we use $\rho_{n}\leq \frac{1}{ b_{\epsilon}}$ and 
$F_{k,0}=1$, $\forall k \in \mathcal{K}$, 
and (b) is valid due to Assumption \ref{assumption:policies}.  \par 
Using \eqref{eq:F1} and the property $\rho_n\leq \frac{1}{ b_{\epsilon}}$, $ M_{k,n}^\theta$ can be upper-bounded by 
\begin{align}
     M_{k,n}^\theta \leq 1+\frac{\zeta\gamma/ b_{\epsilon}}{1-\gamma/ b_{\epsilon}}=\frac{1-(1-\zeta)\gamma/ b_{\epsilon}}{1-\gamma/ b_{\epsilon}}  \triangleq B_{M}^{\theta}
\end{align}
Similarly, we can show that 
\begin{align}
\label{eq:M2}
     \widehat{M}_{k,n}^\theta \leq B_{M}^{\theta}
\end{align}
\end{proof}
\begin{lemma}
\label{lemma:omega_bound2}
For all agents $k\in \mathcal{K}$, $\|\omega_{k,n}\|$ and $\|\widehat{\omega}_{k,n}\|$ can be upper-bounded by $B_n^{\omega}$ defined as 
    \begin{align}
    \label{eq:def_B_omega}
       B_n^{\omega}\triangleq \sum\limits_{i=0}^{n-1} \Omega^{n-i} \frac{\beta_{i}R_{\max}  B_e \|\omega_{0, \max}\|}{ b_{\epsilon}}
    \end{align}  \QEDA
\end{lemma}
\begin{proof}
We start from the update for $\omega_{k,n}$ :
\begin{align}
        \omega_{k,n}=\sum\limits_{\ell\in \mathcal{N}_k}c_{k,\ell}\left(\omega_{\ell, n-1}+\beta_{n}\rho_{\ell,n}\delta_{\ell,n}e_{\ell,n}\right)
\end{align}
Using the definition of $\delta_{\ell, n}$ and the reward upper-bound $r_{\ell,n}\leq R_{\max}$, we can upper-bound $\omega_{k,n}$ as 
\begin{align}
    \|\omega_{k,n}\|&\leq\sum\limits_{\ell\in \mathcal{N}_k}c_{k,\ell}\left(\|\omega_{\ell, n-1}\|\right.\nonumber\\
    &+\|(\gamma \eta_{k,n}- \mu_{k,n})\beta_{n}\rho_{\ell,n} e_{\ell,n}\|\|\omega_{\ell, n-1}\| \nonumber\\
        &\left.+\beta_{n}\rho_{\ell,n}R_{\max} |e_{\ell,n}|\right) \nonumber\\
        &\stackrel{\text{(a)}}{\leq }\sum\limits_{\ell\in \mathcal{N}_k}c_{k,\ell}\left(\left(1+(1+\gamma)\frac{\beta_{0} B_e}{ b_{\epsilon}} \right)\|\omega_{\ell, n-1}\|\right.\nonumber\\
    &\left.+\beta_{n}\frac{\beta_{n}R_{\max}B_e}{ b_{\epsilon}}\right) \nonumber\\
        &\stackrel{\text{(b)}}{\leq }\sum\limits_{i=0}^{n-1} \Omega^{n-i} \frac{\beta_{i}R_{\max}  B_e \|\omega_{0, \max}\|}{ b_{\epsilon}}\triangleq B^{\omega}_n  
\end{align}
 where $\|\omega_{\max}\|=\max\limits_{k\in \mathcal{K}} \|\omega_{k,0}\|$, (a) is valid due to the assumption $\beta_0\geq \beta_{n}$ and $\|\mu_{k,n}\|\leq 1$, $\forall n\geq 0$  and (b) is obtained by iteratively applying inequality (a) and using the definition of $\Omega$. Similarly, we can show that
 \begin{align}
     \|\widehat{\omega}_{k,n}\|\leq B^{\omega}_n
 \end{align}
\end{proof}
\begin{lemma}
\label{lemma:delta_bound}
For all agents $k\in \mathcal{K}$, 
$\|\delta_{k,n}\|$  and $\|\widehat{\delta}_{k,n}\|$  are upperbounded by $B_{n}^{\delta}$ defined as
 \begin{align}
 B_{n}^{\delta}\triangleq R_{\max}+(1+\gamma) B^\omega_{n}
 \end{align}
 where $B^{\omega}_n$ is given in \eqref{eq:def_B_omega}.\par  \QEDA
\end{lemma}
\begin{proof}
We start from the definition of $\delta_{k,n}$
\begin{align}
\delta_{k,n}&= r_{k,n}+(\gamma\eta_{k,n}-\mu_{k,n})^T \omega_{k,n}\nonumber\\
&\stackrel{\text {(a) }}{\leq} R_{\max}+(1+\gamma)\|\omega_{k,n}\|\nonumber\\
&\stackrel{\text {(b)}}{\leq} R_{\max}+(1+\gamma)B_{n}^{\omega} 
\end{align}
where (a) is due to the assumption $\|\mu_{k,n}\|\leq 1$ and (b) is due to Lemma \ref{lemma:omega_bound2}.
Similarly, we can show  that
\begin{align}
\widehat{\delta}_{k,n} \leq B_{n}^{\delta}
\end{align}
\end{proof}
\begin{lemma}
\label{lemma:e2}
$\Delta e_{\ell,i}$ can be bounded as 
\begin{align}
    \|\Delta e_{\ell,i}\|  \leq \frac{\varepsilon  b_{\epsilon}^2\beta_0(1+\gamma)}{ \beta_n \beta_i B_i^{\delta}\Omega^{n-i}\Phi_n} 
\end{align}
if  for all $j \leq i$
    \begin{align}
    \|\Delta \mu_{k,j}\|&\leq  \frac{\varepsilon b_{\epsilon}^2\beta_0(1+\gamma)}{2\beta_n (\gamma \lambda)^{i-j}i|M_{k,j}| \beta_i B^{\delta}_{i}\Omega^{n-i}\Phi_n} \\
     \Delta \rho_{k,j} &\leq  \left(\frac{ b_{\epsilon}}{\gamma}\right)^{i-j}\frac{\varepsilon b_{\epsilon}\beta_0(1-\lambda)^{-1}(1+\gamma)}{2\beta_nF_{k,j}i^2\beta_i B^{\delta}_{i}\Omega^{n-i}\Phi_n} \label{eq:inter112}
    \end{align} \QEDA
\end{lemma}
\begin{proof}
We start by extending the definition of   $\Delta e_{k,i}$:
   \begin{align}
        \Delta e_{k,i} &= \gamma \lambda \Delta e_{k,i-1}+\widehat{M}_{k,i} \widehat{\mu}_{k,i}- M_{k,i} \mu_{k,i}\nonumber\\
        &= \gamma \lambda \Delta e_{k,i-1}+\Delta M_{k,n} \widehat{\mu}_{k,i}+M_{k,i} \Delta \mu_{k,i} \nonumber\\
        &= \sum_{j=0}^{i} (\gamma \lambda)^{i-j} (\Delta M_{k,j} \widehat{\mu}_{k,j}+M_{k,j} \Delta \mu_{k,j})
        \end{align}
    %
Next, we impose the following constraints $\forall j \leq i$:
          \begin{align}
     \|\Delta M_{k,j}\|&\leq  \frac{\varepsilon b_{\epsilon}^2\beta_{0}(1+\gamma)}{2\beta_n(\gamma \lambda)^{i-j}i\beta_i B^{\delta}_i\Omega^{n-i}\Phi_n} \label{eq:delta_M}\\
         \|\Delta \mu_{k,j}\|&\leq  \frac{\varepsilon b_{\epsilon}^2\beta_0(1+\gamma)}{2\beta_n(\gamma \lambda)^{i-j}i M_{k,j}\beta_i B^{\delta}_{ i }\Omega^{n-i}\Phi_n}
    \end{align}
By definition $\Delta M_{k,j} = (1-\lambda) \Delta F_{k,j}$. We can extend $\Delta F_{k,j}$ as  
 \begin{align}
    \label{eq:int1}
        \Delta F_{k,j}&=\gamma(\widehat{\rho}_{k,j-1}\widehat{F}_{k,j-1}-\rho_{k,j-1}F_{k,j-1})\nonumber\\
        &= \gamma(\widehat{\rho}_{k,j-1}\Delta F_{k,j-1}+\Delta \rho_{k,j-1} F_{k,j-1})  \nonumber \\
        &= \sum \limits_{t=0}^{j-1} \prod_{z=t+1}^{j-1} (\gamma\rho_{k,z}) \gamma \Delta \rho_{k,t} F_{k,t}\nonumber\\
        &\leq\sum \limits_{t=0}^{j-1}  \left(\frac{\gamma}{ b_{\epsilon}}\right)^{j-t-1} \gamma \Delta \rho_{k,t}F_{k,t} 
    \end{align}
 where the last inequality is due to $\rho_{k,t}\leq \frac{ 1}{ b_{\epsilon}}$. Therefore, 
\begin{align}
   \Delta M_{k,j} &=(1-\lambda)\sum \limits_{t=0}^{j-1}  \left(\frac{\gamma}{ b_{\epsilon}}\right)^{j-t-1} \gamma \Delta \rho_{k,t}F_{k,t} \nonumber\\
   &\leq \frac{\varepsilon  b_{\epsilon}^2\beta_0(1+\gamma)}{2\beta_n (\gamma \lambda)^{i-j}i\beta_i B^{\delta}_{i }\Omega^{n-i}\Phi_n}
\end{align}
\begin{align}
 \Delta \rho_{k,t} \leq  \left(\frac{ b_{\epsilon}}{\gamma}\right)^{j-t}\frac{\varepsilon  b_{\epsilon}\beta_0(1+\gamma)(1-\lambda)^{-1}}{2 \beta_n F_{k,t} j (\gamma \lambda)^{i-j}i\beta_i B^{\delta}_{i }\Omega^{n-i}\Phi_n} \label{eq:inter111}
\end{align}
The upper-bound  in \eqref{eq:inter111} is minimized if  $j=i$. Therefore, the condition can be generalized $\forall t \leq i$:
\begin{align}
 \Delta \rho_{k,t} \leq  \left(\frac{ b_{\epsilon}}{\gamma}\right)^{i-t}\frac{\varepsilon  b_{\epsilon}(1+\gamma)\beta_0(1-\lambda)^{-1}}{2\beta_n F_{k,t}i^2\beta_i B^{\delta}_{ i }\Omega^{n-i}\Phi_n} \label{eq:inter112}
\end{align}
\end{proof}
\begin{lemma}
\label{theorem1}
For all $k\in \mathcal{K}$, the critic error $\Delta \omega_{k,i}$ can be upper-bounded by  
\begin{align}
    \Delta \omega_{k,i}\leq \frac{3  b_{\epsilon} \varepsilon}{4\beta_i (1+\gamma) i^2 \pi^2 B^{\theta}_{M}} \label{eq:omega_bound}
\end{align}
if $\forall j \leq i$
\begin{align}
    \|\Delta \mu_{\ell ,j})\|&\leq (B_1, B_2) \text{ and } \|\Delta \rho_{\ell,j}\|  &\leq (D_1, D_2) 
\end{align}
where 
\begin{align} 
    B_1 &\triangleq \frac{\varepsilon  b_{\epsilon}}{(1+\gamma)B^{\omega}_{\ell,j}  B_e \beta_j \Omega^{i-j}\Phi_i} \\ 
     B_2&\triangleq\frac{\varepsilon}{2\beta_0 (\gamma \lambda)^{I_1-j}I_1|M_{k,j}|  B_{i}^{\delta}\Omega^{i-I_1}\Phi_i} \\
D_1 &\triangleq \frac{\varepsilon  b_{\epsilon}}{ B_e\beta_j B_j^{\delta}\Omega^{i-j}\Phi_i}   \\
 D_2 &\triangleq \left(\frac{ b_{\epsilon}}{\gamma}\right)^{I_2-j}\frac{\varepsilon   (1-\lambda)^{-1}}{2 F_{k,j}(I_2)^2\beta_0 B_{ i }^{\delta}\Omega^{i-I_2}\Phi_i} \\
  I_1&\triangleq \left[\ln\frac{\Omega}{\gamma \lambda }\right]^{-1}, \quad 
   I_2\triangleq 2 \left[\ln \frac{ b_{\epsilon}\Omega}{\gamma}\right]^{-1} 
\end{align}
\QEDA
\end{lemma}
\begin{proof}
\begin{align}
\Delta \omega_{k,i+1}&=\sum\limits_{\ell \in \mathcal{N}_{k}} c_{k,\ell}\left( \Delta \omega_{\ell,i}+ \beta_i \widehat{\delta}_{\ell,i} \widehat{\rho}_{i} \widehat{e}_{i} \right)\nonumber\\
&-\sum\limits_{\ell \in \mathcal{N}_{k}} c_{k,\ell}\left(\beta_i \delta_{\ell,i} \rho_{\ell,i} e_{\ell,i} \right) \nonumber\\
&= \sum\limits_{\ell \in \mathcal{N}_{k}} c_{k,\ell}\left( \Delta \omega_{\ell,i-1}+ \beta_i \Delta \delta_{\ell,i} \widehat{\rho}_{i} \widehat{e}_{i} \right.\nonumber\\
&\left.+\beta_i \delta_{\ell,i} \Delta \rho_{\ell,i} \widehat{e}_{i} +\beta_i \delta_{\ell,i} \rho_{\ell,i} \Delta e_{\ell,i} \right)  \label{eq:inter1111}
\end{align}
We can express $\Delta \delta_{k,i}$ as 
\begin{align}
\|\Delta \delta_{k,i}\|&= \|(\gamma \widehat{\eta}_{k,i}- \widehat{\mu}_{k,i}) \widehat{\omega}_{k,i}- (\gamma \eta_{k,i}- \mu_{k,i})\omega_{k,i}\| \nonumber\\
&= \|(\gamma \widehat{\eta}_{k,i}- \widehat{\mu}_{k,i})\Delta\omega_{k,i}\|+\| (\gamma \Delta \eta_{k,i}- \Delta \mu_{k,i}) \omega_{k,i}\|\nonumber\\
&\leq (1+\gamma)\|\Delta\omega_{k,i}\|+ \|(\gamma \Delta \eta_{k,i}- \Delta \mu_{k,i})\|\| \omega_{k,i}\|
\end{align}
Therefore, \eqref{eq:inter1111} can be modified as 

\begin{align}
\label{eq:inter_w}
&\|\Delta \omega_{k,i+1}\|=\sum\limits_{\ell\in \mathcal{N}_k} c_{\ell,i}\left( \|\Delta \omega_{\ell,i}\|+ \beta_i |\Delta \delta_{\ell,i}|\widehat{\rho}_{i}\|\widehat{e}_{i}\| \right.\nonumber\\
&\left.+\beta_i |\delta_{\ell,i}| \|\Delta e_{\ell,i}\| \widehat{\rho}_{\ell,i} +\beta_i |\delta_{\ell,i}| \|e_{\ell,i}\| \Delta \rho_{\ell,i} \right) \nonumber \\
&\leq \left(1+\beta (1+\gamma) \widehat{\rho}_{i}\|\widehat{e}_{i}\| \right) \sum\limits_{\ell \in \mathcal{N}_k}  c_{\ell, i }\|\Delta \omega_{\ell,i}\| \nonumber\\
&+\sum\limits_{\ell\in \mathcal{N}_k} c_{\ell, i }\left(\beta_i |(\gamma \Delta \eta_{\ell ,i}- \Delta \mu_{\ell ,i})|\| \omega_{\ell,i}\| \widehat{\rho}_{i} \|\widehat{e}_{i}\| \right. \nonumber \\
&\left. +\beta_i |\delta_{\ell,i}| \|\Delta e_{\ell,i}\| \widehat{\rho}_{i} +\beta_i \|\delta_{\ell,i}\| \|e_{\ell,i}\| \|\Delta \rho_{\ell,i} \|\right)   \nonumber\\
&\leq \left(1+\beta (1+\gamma) \widehat{\rho}_{i}\|\widehat{e}_{i}\| \right) \sum\limits_{\ell \in \mathcal{N}_k}  c_{\ell, i }\|\Delta \omega_{\ell,i}\|\nonumber\\
&+ \sum\limits_{\ell\in \mathcal{N}_k} c_{\ell, i }\left(\beta_i \|(\gamma \Delta \eta_{\ell ,i}- \Delta \mu_{\ell ,i})\|\| \omega_{\ell,i}\| \widehat{\rho}_{i} \|\widehat{e}_{i}\|  \right. \nonumber \\
&\left. +\beta_i \|\delta_{\ell,i}\| |\Delta \rho_{\ell,i}| \|\widehat{e}_{i}\| +\beta_i |\delta_{\ell,i}| \rho_{\ell,i} \|\Delta e_{\ell,i} \|\right)
\end{align}
Now, we impose the following constraint $\forall \ell \in \mathcal{K}$, $\forall j\leq i$:
\begin{align}
&\left(\beta_j \|(\gamma \Delta \eta_{\ell ,j}- \Delta \mu_{\ell ,j})\|\| \omega_{\ell,j}\| \widehat{\rho}_{\ell,j} \|\widehat{e}_{\ell,j}\| \right.\nonumber\\
&\left.+\beta_j |\delta_{\ell,j}| \|\Delta \rho_{\ell,j}\| \|\widehat{e}_{\ell,j}\| +\beta_j |\delta_{\ell,j}| \rho_{\ell,j} \|\Delta e_{\ell,j} \|\right)\nonumber\\
&\leq \frac{3\varepsilon  b_{\epsilon} \Omega^{j-i}}{4\beta_i (1+\gamma)\pi^2 B^{\theta}_{M} i^3 }
\end{align}
Since the left-hand side comprises a summation of three terms, we assert that each term is constrained to be at most one-third of the value on the right-hand side. Next, we use these inequalities  to explicitly express bounds for $\|\Delta \mu_{k,j}\|$, $\|\Delta \eta_{\ell,j}\|$, $ \|\Delta e_{\ell,j}\|$ and $ \|\Delta \rho_{\ell,j}\|$:
\begin{align}
& \|\Delta \eta_{\ell ,j}\| \leq \frac{\varepsilon b_{\epsilon}(1+\gamma)^{-2}\pi^{-2}}{4\beta_i B^{\theta}_{M} i ^3 \|\omega_{\ell,j}\| \widehat{\rho}_{\ell,j} \|\widehat{e}_{\ell,j}\|\beta_j \Omega^{i-j}} \label{eq:detla_eta11}\\ 
& \|\Delta \mu_{\ell ,j})\|\leq \frac{\varepsilon  b_{\epsilon} (1+\gamma)^{-2}\pi^{-2}}{4\beta_i  B^{\theta}_{M} i ^3\|{\omega}_{\ell,j}\| \widehat{\rho}_{\ell,j} \|\widehat{e}_{\ell,j}\beta_j \Omega^{i-j}} \\ 
&\|\Delta e_{\ell,j}\|  \leq \frac{\varepsilon b_{\epsilon} (1+\gamma)^{-1} \pi^{-2}}{4 \beta_i  B^{\theta}_{M} i ^3\|\widehat{\rho}_{\ell,j}\|\beta_j |\delta_{\ell, j }|\Omega^{i-j}}\\
&\|\Delta \rho_{\ell,j}\|  \leq \frac{\varepsilon  b_{\epsilon}(1+\gamma)^{-1} \pi^{-2}}{4 \beta_i  B^{\theta}_{M} i ^3\|e_{\ell,j}\|\beta_j |\delta_{\ell, j }|\Omega^{i-j}} \label{eq:delta_rho11} 
\end{align}
   If \eqref{eq:detla_eta11}-\eqref{eq:delta_rho11} are satisfied then \eqref{eq:inter_w} can be upper-bounded as
\begin{align}
&\|\Delta \omega_{k,i+1}\| \nonumber\\
&\leq \sum \limits_{\ell \in \mathcal{N}_k}\left[\sum_{j=0}^{i} \prod\limits_{z=j+1}^i  \frac{3c_{k,
\ell} \varepsilon  b_{\epsilon}\left(1+\beta (1+\gamma) \widehat{\rho}_{z}\|\widehat{e}_{z}\| \right)}{8(1+\gamma)\pi^2 B^{\theta}_{M} i^3 \Omega^{i-j}}\right] \nonumber \\
&\leq \sum \limits_{\ell \in \mathcal{N}_k}\frac{3c_{k, \ell}\varepsilon b_{\epsilon}(1+\gamma)^{-1}}{4\beta_i i^2\pi^2B^{\theta}_{M}}=\frac{3\varepsilon b_{\epsilon}(1+\gamma)^{-1}}{4\beta_ii^2\pi^2B^{\theta}_{M}}
\end{align}
To make conditions \eqref{eq:detla_eta11}-\eqref{eq:delta_rho11}  interpretable we  bound some terms in the expressions using Lemmas \ref{lemma:e_bound} and \ref{lemma:delta_bound} and $\rho_{\ell, j}\leq \frac{1}{ b_{\epsilon}}$ (Assumption \ref{assumption:policies}):
\begin{align}
& \|\Delta \eta_{\ell ,j}\| \leq \frac{\varepsilon  b_{\epsilon}^2(1+\gamma)^{-2}\pi^{-2}}{4\beta_{i} B^{\theta}_{M} i ^3 B^{\omega}_{j}  B_e \beta_j \Omega^{i-j}} \\ 
& \|\Delta \mu_{\ell ,j})\|\leq \frac{\varepsilon  b_{\epsilon}^2(1+\gamma)^{-2}\pi^{-2}}{4 \beta_i  B^{\theta}_{M} i ^3B^{\omega}_{j}  B_e \beta_j \Omega^{i-j}}  \label{eq:mu_inter} \\ 
&\|\Delta e_{\ell,j}\|  \leq \frac{\varepsilon  b_{\epsilon}^2(1+\gamma)^{-1}\pi^{-2}}{4 \beta_i  B^{\theta}_{M} i ^3\beta_j B_j^{\delta}\Omega^{i-j}} \label{eq:inter_e22}\\
&\|\Delta \rho_{\ell,j}\|  \leq \frac{\varepsilon  b_{\epsilon}(1+\gamma)^{-1}\pi^{-2}}{4 \beta_{i} B^{\theta}_{M} i ^3 B_e\beta_j B_j^{\delta}\Omega^{i-j}} \label{eq:rho_inter}
\end{align}
Next, using Lemma \ref{lemma:e2} we express the constraint \eqref{eq:inter_e22} in terms of the conditions for  $\Delta \rho_{k,i}$ and $\Delta \mu_{k,i}$. Namely, \eqref{eq:inter_e22} is satisfied if  for all $z \leq j$
    \begin{align}
    &\|\Delta \mu_{k,z}\|\leq  \frac{\varepsilon b_{\epsilon}^2\beta_0}{2\beta_i(\gamma \lambda)^{j-z}j|M_{k,z}|  \beta_j B^{\delta}_{j}\Omega^{i-j}\Phi_i} \label{eq:inter213}\\
     &\Delta \rho_{k,z} \leq  \left(\frac{ b_{\epsilon}}{\gamma}\right)^{j-z}\frac{\varepsilon b_{\epsilon}(1-\lambda)^{-1}\beta_0}{2 \beta_i F_{k,z}i^2  \beta_j B^{\delta}_{j}\Omega^{i-j}\Phi_i} \label{eq:inter113}
    \end{align}
Now, we minimize the upperbounds in \eqref{eq:inter213} and  \eqref{eq:inter113} with respect to $j$, so we can generalize the conditions 
\begin{align}
       \|\Delta \mu_{k,z}\|\leq  \min \limits_{j\leq i}\frac{\varepsilon b_{\epsilon}\beta_0(\gamma \lambda)^{z-j}}{2 \beta_i j|M_{k,z}| \beta_j B_{j }^{\delta} \Omega^{i-j}\Phi_n}
\end{align}
which is equivalent to solving     
\begin{align}
    &\max\limits_{j\leq i} \ (\gamma \lambda)^{j}j \beta_j B_{ j }^{\delta}\Omega^{-j}
    \stackrel{\text {(a)}}{\leq} B_i^{\delta}\beta_0 \max\limits_{j\leq i} \ (\gamma \lambda)^{j}j \Omega^{-j}
\end{align}
where in  (a), we leverage the property that $B^\delta_{j}$ increases with $j$ (see Lemma \ref{lemma:delta_bound}) and the assumption $\beta_0 \geq \beta_{j}$ for all $j \geq 0$.
Therefore, we seek to solve the optimization problem 
\begin{align}
 \max\limits_{j\leq i} \ \left(\frac{\gamma \lambda}{\Omega}\right)^{j}j
\end{align}
  which is maximized if
  \begin{align}
  \label{eq:I1}
  j=\left[\ln\frac{\Omega}{\gamma \lambda }\right]^{-1}\triangleq I_1
  \end{align}
 Next, we minimize the upperbound in \eqref{eq:inter113}  with respect to $j$
\begin{align}
  \min \limits_{j\leq i} \left(\frac{ b_{\epsilon}}{\gamma}\right)^{j-z-1}\frac{\varepsilon b_{\epsilon}^2(1-\lambda)^{-1}\pi^{-2}}{8 \beta_i\gamma F_{k,z}j^2 B^{\theta}_{M} i ^3 \beta_j B^{\delta}_{j}\Omega^{i-j}}
\end{align}
which is equivalent to solving 
\begin{align}
    &\max \limits_{j\leq i} \left(\frac{\gamma}{ b_{\epsilon}}\right)^{j} j^2 \Omega^{-j} \beta_j B_j^{\delta }\leq  B_i^{\delta }\beta_0\max \limits_{j\leq i} \left(\frac{\gamma}{ b_{\epsilon}\Omega}\right)^{j} j^2 
\end{align}
This is maximized for 
\begin{align}
\label{eq:I2}
j=2 \left[\ln (\frac{ b_{\epsilon}}{\gamma}+\frac{\beta (1+\gamma)B_e}{\gamma})\right]^{-1} \triangleq I_2
\end{align}
Combining \eqref{eq:mu_inter}, \eqref{eq:rho_inter}-\eqref{eq:inter113}, \eqref{eq:I1} and \eqref{eq:I2}  we reach  \eqref{eq:omega_bound}.
\end{proof}
\begin{lemma}
For all agents $k\in \mathcal{}K$ and $\forall i\leq (n-1)$, it holds that 
\label{lemma:delta_Psi_bound}
    \begin{align}
    \label{eq:delta_Psi_bound}
        \|\Delta \Psi_{k,i}\|\leq \frac{3\varepsilon b_{\epsilon}}{2\beta_i i^2 \pi^2 B_M^{\theta} B^{\delta}_{i} }   
    \end{align} \QEDA
\end{lemma}
\begin{proof}
\begin{align}
    \|\Delta \Psi_{k,i}(a)\|&= \|\Delta \mu_{k,i}- \widehat{\mu}_{k,i} \widehat{\pi} (a|\widehat{\mu}_{k,i})+\mu_{k,i} \pi (a|\mu_{k,i})\|\nonumber\\
    &=\|\Delta \mu_{k,i}- \Delta \mu_{k,i} \widehat{\pi}_{k,i} (a\|\widehat{\mu}_{k,i})\nonumber\\
    &-\mu_{k,i} \Delta \pi_{k,i} (a\|\mu_{k,i})\| \nonumber\\
    &\leq\|\Delta \mu_{k,i}\|+ \|\Delta \mu_{k,i} \widehat{\pi}_{k,i} (a\|\widehat{\mu}_{k,i})\|\nonumber\\
    &+\|\mu_{k,i}\| \|\Delta \pi_{k,i} (a\|\mu_{k,i})\|\nonumber\\
    &\stackrel{\text { (a) }}{\leq} (2+L) \Delta \mu_{k,i}
\end{align}
where in (a) we use Lipschitz continuity of policy functions with respect to belief vectors.
Therefore, we can claim that \eqref{eq:delta_Psi_bound} is satisfied if 
\begin{align}
\Delta \mu_{k,i} \leq \frac{3\varepsilon b_{\epsilon}}{2(2+L)\beta_i i^2 \pi^2 B_M^{\theta} B^{\delta}_{i} } 
\end{align}
\end{proof}
\subsection{Proof of Theorem \ref{theorem2}}
\begin{proof}
       \begin{align}
       \label{eq:theta_bound}
        \Delta \theta_{k,n}&=\Delta \theta_{k,n-1}+\beta_n\widehat{\rho}_n \widehat{M}^{\theta}_n \widehat{\Psi}_{n} \widehat{\delta}_{k,n}\nonumber\\
        &- \beta_n\rho_{k,n} M^{\theta}_{k,n} \Psi_{k,n} \delta_{k,n}\nonumber\\
        &= \Delta \theta_{k,n-1}
        +\beta_n(\Delta \rho_{k,n} \widehat{M}^{\theta}_n \widehat{\Psi}_{n} \widehat{\delta}_{k,n}\nonumber\\
        &+ \rho_{k,n} \Delta M^{\theta}_{k,n} \widehat{\Psi}_{k,n} \widehat{\delta}_{k,n}
        +\rho_{k,n}  M^{\theta}_{k,n} \Delta \Psi_{k,n} \widehat{\delta}_{k,n} \nonumber\\
        &+ \rho_{k,n}  M^{\theta}_{k,n}  \Psi_{k,n} \Delta \delta_{k,n} )\nonumber\\
        &=\sum\limits_{i=0}^{n-1} \beta_i\left(\Delta \rho_{k,i} \widehat{M}^{\theta}_i \widehat{\Psi}_{i} \widehat{\delta}_{k,i}
        +\rho_{k,i} \Delta M^{\theta}_{k,i} \widehat{\Psi}_{k,i} \widehat{\delta}_{k,i}\right.\nonumber\\
        &\left.+\rho_{k,i}  M^{\theta}_{k,i} \Delta \Psi_{k,i} \widehat{\delta}_{k,i} +  \rho_{k,i}  M^{\theta}_{k,i}  \Psi_{k,i} \Delta \delta_{k,i}\right)     \end{align}
where, in obtaining the last equality, we repeatedly apply the recursive formula. We aim to bound \eqref{eq:theta_bound} by $\varepsilon$. Given the well-known series:
\begin{align}
\sum\limits_{i=1}^{\infty}\frac{1}{i^2} = \frac{\pi^2}{6}
\end{align}
we introduce the inequalities:
\begin{align}
\Delta \rho_{k,i}  \leq \frac{3}{2\beta_i\pi^2i^2\widehat{M}^{\theta}_i \widehat{\Psi}_{i} \widehat{\delta}_{k,i}} \label{eq:rho_bound_inter1}\\
 \Delta M^{\theta}_{k,i} \leq \frac{3}{2\beta_i\pi^2i^2 \rho_{k,i}\widehat{\Psi}_{k,i}\widehat{\delta}_{k,i}} \\
 \Delta \Psi_{k,i} \leq \frac{3}{2\beta_i\pi^2i^2\rho_{k,i} M^{\theta}_{k,i}\widehat{\delta}_{k,i} } \label{eq:Psi_bound_inter1}\\ 
 \Delta \delta_{k,i}\leq \frac{3}{2\beta_i\pi^2i^2\rho_{k,i} M^{\theta}_{k,i} \Psi_{k,i}} \label{eq:delta_bound_inter1}
\end{align}
Then, substituting \eqref{eq:rho_bound_inter1}-\eqref{eq:delta_bound_inter1} into \eqref{eq:theta_bound}, we obtain:
\begin{align}
\Delta \theta_{k,n}&\leq \sum\limits_{i=0}^{\infty} \beta_i\left(\Delta \rho_{k,i} \widehat{M}^{\theta}_i \widehat{\Psi}_{i} \widehat{\delta}_{k,i}\right. \nonumber\\
&+\rho_{k,i} \Delta M^{\theta}_{k,i} \widehat{\Psi}_{k,i} \widehat{\delta}_{k,i}+\rho_{k,i} M^{\theta}_{k,i} \Delta \Psi_{k,i} \widehat{\delta}_{k,i} \nonumber\\
&\left.+ \rho_{k,i} M^{\theta}_{k,i} \Psi_{k,i} \Delta \delta_{k,i}\right) \leq \frac{6}{\pi^2}\sum\limits_{i=1}^{\infty} \frac{1}{i^2}= \varepsilon
\end{align}
We aim to write the conditions in \eqref{eq:rho_bound_inter1}-\eqref{eq:delta_bound_inter1} in terms of conditions for $\Delta\rho_{k,i}$ and $\Delta\mu_{k,i}$.  To this end, we first  express $\Delta \delta_{k,i}$ as 
\begin{align}
|\Delta \delta_{k,i}|&= |(\gamma \widehat{\eta}_{k,i}- \widehat{\mu}_{k,i}) \widehat{\omega}_{k,i}- (\gamma \eta_{k,i}- \mu_{k,i})\omega_{k,i}| \nonumber\\
&= |(\gamma \widehat{\eta}_{k,i}- \widehat{\mu}_{k,i})\Delta\omega_{k,i}+ (\gamma \Delta \eta_{k,i}- \Delta \mu_{k,i}) \omega_{k,i}|\nonumber\\
&\leq (1+\gamma)\|\Delta\omega_{k,i}\|\nonumber\\
&+ \|(\gamma \Delta \eta_{k,i}- \Delta \mu_{k,i})\|\| \omega_{k,i}\|
\end{align}
Condition \eqref{eq:delta_bound_inter1} can be split into the following tighter conditions  $\forall i\leq n-1$ 
\begin{align}
            &\Delta \omega_{k,i}\leq \frac{3\varepsilon}{4\beta_i (1+\gamma)i^2 \pi^2M^{\theta}_{k,i} {\Psi}_{k,i}\rho_{k,i} }  \label{inter:omega2}\\
         &\Delta \mu_{k,i}\leq \frac{3\varepsilon}{4\beta_i (1+\gamma)i^2 \pi^2M^{\theta}_{k,i} \Psi_{k,i}\rho_{k,i} \|\omega_{k,i}\| } \label{inter:mu2}
\end{align}
Note that $\rho_{k,i}, M^{\theta}_{k,i}, \widehat {M}^{\theta}_{k,i}, \Psi_{k,i}, \widehat{\Psi}_{k,i}$, $\omega_{k,i}$, and $\widehat{\delta}_{i}$ are not known at the beginning of iteration $i$. To circumvent this issue,  we can upperbound these variables and get tighter conditions  than in (\ref{eq:rho_bound_inter1}-\ref{eq:Psi_bound_inter1}) and (\ref{inter:omega2})-(\ref{inter:mu2}). The importance sampling ratio $\rho_{k,i}$ can be upper-bounded by $\frac{1}{ b_{\epsilon}}$ . Using the properties $\|\mu^b_{k,n}\|\leq 1$, $\forall b \in \mathcal{A}$ and $\|\pi(\cdot)\|\leq 1$,  we can upper-bound $\Psi_{k,i}$ as 
\begin{align}
\label{eq:Psi}
\|\Psi_{k,i}(a)\|=\|\mu_{k,i}-\mu_{k,i} \pi(a\|\mu_{k,i}; \theta_{k,i})\|\leq 1
\end{align}
Similarly, we can show that 
\begin{align}
\label{eq:Psi2}
\|\widehat{\Psi}_{k,i}\|\leq 1 
\end{align}
The upper-bounds for $\widehat{\delta}_{i}$, $M^{\theta}_{k,i}$, and  $\widehat {M}^{\theta}_{k,i}$ are given in Lemmas \ref{lemma:delta_bound} and \ref{lemma:M_theta_bound}, respectively. 
Therefore, the conditions in (\ref{eq:rho_bound_inter1})-(\ref{eq:Psi_bound_inter1}) and (\ref{inter:omega2})-(\ref{inter:mu2}) can be modified as 
\begin{align}
     &\|\Delta \omega_{k,i}\|\leq \frac{3  b_{\epsilon} \varepsilon}{4\beta_i (1+\gamma)i^2 \pi^2 B^{\theta}_{M}} \label{inter:omega3}\\
         &\|\Delta \mu_{k,i}\|\leq \frac{3  b_{\epsilon} \varepsilon}{4 \beta_i (1+\gamma)i^2 \pi^2 B^{\theta}_{M} B_{i}^{\omega}} \label{inter:mu3}\\
    &\|\Delta M^{\theta}_{k,i}\|\leq \frac{3\varepsilon b_{\epsilon}}{2\beta_i i^2 \pi^2   B^{\delta}_{i} } \label{inter:M_theta_3}\\
     &\|\Delta \Psi_{k,i}\|\leq \frac{3\varepsilon b_{\epsilon}}{2\beta_i i^2 \pi^2 B_M^{\theta} B^{\delta}_{i} }    \label{eq:Psi3}\\
         &\|\Delta \rho_{k,i}\|\leq \frac{3\varepsilon}{2\beta_i i^2 \pi^2 B_M^{\theta}  B^{\delta}_{i}} \label{eq:rho3}
\end{align}
According to Lemma \ref{theorem1},  \eqref{inter:omega3} is satisfied if 
 $\forall j \leq i \leq n $:
    \begin{align}
    \label{eq:inter_mu111}
        &\Delta \mu_{k,j}\leq \min (B_1, B_2)  \text{ and } 
         \Delta\rho_{k,j} \leq  \min (D_1, D_2)
    \end{align}    
To generalize the results in \eqref{eq:inter_mu111}, i.e. to find the conditions which will be valid for all $j\leq n-1$, we need to minimize the bounds with respect to $i$. We can easily derive that 
\begin{align}
    &\min\limits_{i\leq n} B_1(i)=B_1(n)\triangleq \widetilde{B}_1 \\ &\min\limits_{i\leq n} B_2(i)=B_2(n)\triangleq \widetilde{B}_2 \\
       &\min\limits_{i\leq n} D_1(i)=D_1(n)\triangleq \widetilde{D}_1\\
       &\min\limits_{i\leq n} D_2(i)=D_2(n)\triangleq \widetilde{D}_2 
\end{align}
Next, according to Lemma \ref{lemma:Delta_M_theta_bound}, \eqref{inter:M_theta_3} is met 
if  for all $j\leq (i-1) <n$
\begin{align}
\label{eq:int6}
    \Delta \rho_{k,j} \leq  \left(\frac{ b_{\epsilon}}{\gamma}\right)^{i-j-1} \frac{3 \varepsilon b_{\epsilon}^2}{8\gamma\zeta\beta_n i^2 \pi^2   B^{\delta}_{i} F_{k,j}}
\end{align}
To find the general condition that will satisfy for all $i<n-1$ we need to minimize the bound in \eqref{eq:int6} with respect to $i$. Equivalently, we need to solve the following maximization problem 
 \begin{align}
 \label{eq:inter1111}
 \max \limits_{i\leq n}\left(\frac{\gamma}{ b_{\epsilon} }\right)^i i^2 B_{i}^{\delta} \stackrel{\text { (a) }}{\leq} B_{n}^{\delta}\max \limits_{i\leq n}\left(\frac{\gamma}{ b_{\epsilon} }\right)^i i^2 
 \end{align}
 where (a) is due to the monotonicity of $B^{\delta}_i$. 
 The solution for the last maximization in \eqref{eq:inter1111} is 
 \begin{align}
 i^*=\frac{2}{\ln\frac{ b_{\epsilon}}{\gamma}}\triangleq I_3 
 \end{align}
According to Lemma \ref{lemma:delta_Psi_bound},  \eqref{eq:Psi3} is fulfilled if 
\begin{align}
    \Delta \mu_{k,i}\leq \frac{3\varepsilon b_{\epsilon}}{2(2+L)\beta_n i^2 \pi^2 B_M^{\theta} B^{\delta}_{i} } 
\end{align}
\end{proof}
\begin{lemma}
\label{lemma1}
For all agents $k\in \mathcal{K}$, it holds that 
\begin{equation}
    \|\Delta M_{k,n}\|\triangleq \|\widehat{M}_{k,n}-M_{k,n}\|\leq \varepsilon 
\end{equation}
if  $\forall i \leq (n-1)$
\begin{equation}
    \Delta \rho_{k,i} \leq \frac{6\varepsilon  b_{\epsilon}^{n-i-1}}{\pi^2 i^2 \gamma F_{k,i}\gamma^{n-i-1}(1-\lambda)}
\end{equation} \QEDA
\end{lemma}
\begin{proof}
    \begin{align}
    \label{eq:int1}
        \Delta F_{k,n}&=\widehat{F}_{k,n}-F_{k,n}\nonumber\\
        &=\gamma(\widehat{\rho}_{k,n-1}\widehat{F}_{k,n-1}-\rho_{k,n-1}F_{k,n-1})\nonumber\\
        &= \gamma(\widehat{\rho}_{k,n-1}\Delta F_{k,n-1}+\Delta \rho_{k,n-1} F_{k,n-1})  \nonumber \\
        &= \sum \limits_{i=0}^{n-1} \prod_{j=i+1}^{n-1} (\gamma\rho_{k,j}) \gamma \Delta \rho_{k,i} F_{k,i}\nonumber\\
        &\leq   \sum \limits_{i=0}^{n-1}  \left(\frac{\gamma}{ b_{\epsilon}}\right)^{n-i-1} \gamma \Delta \rho_{k,i}F_{k,i} 
    \end{align}
We know that $\sum \limits_{i=1}^{\infty}\frac{1}{i^2}=\frac{\pi^2}{6}$. To bound  \eqref{eq:int1} by $\varepsilon$ we can impose the following condition: 
\begin{align}
\Delta \rho_{k,i}  \leq \left(\frac{ b_{\epsilon}}{\gamma }\right)^{n-i-1} \frac{6\varepsilon}{\gamma F_{k,i}\pi^2 i^2}
\end{align}
We notice that $\Delta M_{k,n} = (1-\lambda) \Delta F$. So,  $\Delta M_{k,n}$ is $\varepsilon$- bounded  if 
\begin{align}
\Delta \rho_{k,i}  \leq \left(\frac{ b_{\epsilon}}{\gamma}\right)^{n-i-1}\frac{6\varepsilon}{\gamma \pi^2 i^2 F_{k,i}(1-\lambda)}
\end{align}
then \eqref{eq:int1} is $\varepsilon$- bounded at time $n$. 
\end{proof}
\begin{lemma}
\label{lemma:Delta_M_theta_bound}
For all agents $k\in \mathcal{K}$
\begin{align}
\label{eq:Delta_M_theta_bound}
    \Delta M^{\theta}_{k,i} \leq \frac{3\varepsilon b_{\epsilon}}{4\beta_n i^2 \pi^2   B^{\delta}_{i} }
\end{align}
if  for all $j\leq (i-1)$
\begin{align}
    \Delta \rho_{k,j} \leq  \left(\frac{ b_{\epsilon}}{\gamma}\right)^{i-j-1} \frac{3 \gamma\varepsilon b_{\epsilon}^2}{8\gamma\zeta\beta_n i^2 \pi^2   B^{\delta}_{i} F_{k,j}}
\end{align} \QEDA
\end{lemma}
\begin{proof}
\begin{align}
\Delta M^{\theta}_{k,i}&=\lambda^{\theta}\gamma(\widehat{\rho}_{k,i-1}\widehat{F}_{k,i-1}-\rho_{k,i-1}F_{k,i-1})\nonumber\\
&=\lambda^{\theta}\gamma (\Delta \rho_{k,i-1}F_{k,i-1}+\widehat{\rho}_{k,i-1} \Delta F_{k,i-1})\nonumber\\
&\leq\frac{3\varepsilon b_{\epsilon}}{4\beta_n i^2 \pi^2   B^{\delta}_{i} }
\end{align}
The last inequality is true if we impose the following constraints 
\begin{align}
\label{eq:rho2}
   \Delta \rho_{k,i-1} \leq\frac{3\varepsilon b_{\epsilon}}{8\zeta\beta_n i^2 \pi^2   B^{\delta}_{i} F_{k,i-1}} \text{ and }
\end{align}
 \begin{align}
    \label{eq:int_F1}
        \Delta F_{k,i} \leq \frac{3\varepsilon b_{\epsilon}^2}{8\zeta\beta_n i^2 \pi^2   B^{\delta}_{i} }
    \end{align}
 where (a) is due to $\rho_{k,t}\leq \frac{ 1}{ b_{\epsilon}}$.
Next, we extend the definition of  $\Delta F_{k,i}$:
\begin{align}
    \Delta F_{k,i} &=\widehat{F}_{k,i}-F_{k,i}\\
    &=\gamma(\widehat{\rho}_{k,i-1}\widehat{F}_{k,i-1}-\rho_{k,i-1}F_{k,i-1}) \nonumber\\
    &= \gamma(\widehat{\rho}_{k,i-1}\Delta F_{k,i-1}+\Delta \rho_{k,i-1} F_{k,i-1})  \nonumber \\
        &= \sum \limits_{j=0}^{i-1} \prod_{z=j+1}^{i-1} (\gamma\rho_{k,z}) \gamma \Delta \rho_{k,j} F_{k,j} \nonumber\\
        &\stackrel{\text { (a) }}{\leq}\sum \limits_{t=0}^{i-1}  \left(\frac{\gamma}{ b_{\epsilon}}\right)^{i-j-1} \gamma \Delta \rho_{k,j}F_{k,j} 
\end{align}
Therefore, the condition in \eqref{eq:int_F1} can be replaced with the following tighter condition $\forall j \leq (i-1)$:
\begin{align}
\label{eq:int5}
     \Delta \rho_{k,j} \leq  \left(\frac{ b_{\epsilon}}{\gamma}\right)^{i-j-1} \frac{3 \gamma\varepsilon b_{\epsilon}^2}{8\gamma\zeta\beta_n i^2 \pi^2   B^{\delta}_{i} F_{k,j}}
\end{align}
Note that for $j=i-1$ the condition in \eqref{eq:int5} is more strict than in \eqref{eq:rho2}. Hence, we can assert a generalization that \eqref{eq:Delta_M_theta_bound} holds true if \eqref{eq:int5} is satisfied.  
\end{proof}
\begin{lemma}
\label{lemma:e}
At time $n$, $\Delta e_{k,n}$ is $\varepsilon$-bounded if $\forall i \leq n-1$:
\begin{equation}
    \Delta \rho_{k,i} \leq \frac{18\varepsilon  b_{\epsilon}^{n-i-1}}{ n^2 \pi^4 i^2  F_{k,t}\gamma^{n-i}(1-\lambda)}
\end{equation}
and  $\forall i \leq n$
\begin{align}
\Delta \mu_{k,i} \leq \frac{3\varepsilon}{\pi^2} \frac{1}{i^2(\gamma \lambda)^{n-i}\|\|M_{k,i}\|\| } 
\end{align}  \QEDA
\end{lemma}
\begin{proof}
We start from definition of $\Delta e_{k,n}$
      \begin{align}
        \Delta e_{k,n} &= \gamma \lambda \Delta e_{k,n-1}+\widehat{M}_{k,n} \widehat{\mu}_{k,n}- M_{k,n} \mu_{k,n}\nonumber\\
        &= \gamma \lambda \Delta e_{k,n-1}+\Delta M_{k,n} \widehat{\mu}_{k,n}+M_{k,n} \Delta \mu_{k,n} \nonumber\\
        &\leq \varepsilon
    \end{align}
Therefore, we have 
      \begin{align}
        \Delta e_{k,n} &= \sum_{i=0}^{n} \prod_{j=i+1}^{n} (\gamma \lambda)^{j} (\Delta M_{k,i} \widehat{\mu}_{k,n}+M_{k,i} \Delta \mu_{k,i}) \nonumber\\
        &\leq \frac{6\varepsilon}{\pi^2}\sum_{i=0}^{n} \frac{1}{i^2} \leq \frac{6\varepsilon}{\pi^2}\sum_{i=0}^{\infty} \frac{1}{i^2} = \varepsilon 
    \end{align}
So, we have that $\Delta e_{k,n}$ is $\varepsilon$-bounded if  $\forall i\leq n$:
\begin{align}
 \|\Delta M_{k,i}\|  &\leq \frac{3\varepsilon}{\pi^2} \frac{1}{i^2(\gamma \lambda)^{n-i}\|\widehat{\mu}_{k,i}\|} \nonumber\\
 &\leq \frac{3\varepsilon}{\pi^2} \frac{1}{i^2(\gamma \lambda)^{n-i}}
\end{align}
\begin{align}
\Delta \mu_{k,i} \leq \frac{3\varepsilon}{\pi^2} \frac{1}{i^2(\gamma \lambda)^{n-i}M_{k,i} } 
\end{align}
By Lemma \ref{lemma1},  $ \Delta M_{k,i}$ is bounded  by 
\begin{align}
    \frac{3\varepsilon}{\pi^2} \frac{1}{i^2(\gamma \lambda)^{n-i}}
\end{align}
if for all $t\leq (i-1)$:
\begin{align}
\label{eq:inter2}
    \Delta \rho_{k,t} &\leq     \frac{18\varepsilon  b_{\epsilon}^{i-t-1}(1-\lambda)^{-1}}{\pi^2 i^2(\gamma \lambda)^{n-i} \pi^2 t^2 \gamma F_{k,t}\gamma^{i-t-1}}\\
    &= \frac{18\varepsilon  b_{\epsilon}^{-t-1} ( b_{\epsilon} \lambda)^i}{\pi^2 i^2(\gamma \lambda)^{n} \pi^2 t^2 \gamma F_{k,t}\gamma^{-t-1}(1-\lambda)}
\end{align}
Since $ b_{\epsilon} \lambda\leq 1$ the bound in  \eqref{eq:inter2} is minimized  for $i=n$ 
\begin{equation}
    \Delta \rho_{k,t} \leq \frac{18\varepsilon  b_{\epsilon}^{n-t-1}}{ n^2 \pi^4 t^2  F_{k,t}\gamma^{n-t}(1-\lambda)}
\end{equation}
\end{proof}

\subsection{Proof of Corollary \ref{corollary1}}
\begin{proof}
We start from the definition of $ \widehat{\rho}_{k,n}$:
    \begin{align}
      \widehat{\rho}_{k,n}(a) 
         &=\frac{\pi_{k,n}(a\|\widehat{\mu}_{k,n})}{b_{k,n}(a\|\widehat{\mu}_{k,n})}-\rho_{k,n}(a)\nonumber\\
    &=\frac{\exp[\widehat{\mu}_{k,n}^T (\theta_a-\bar{\theta}_a)] \sum\limits_{c \in \mathcal{A}}\exp[\widehat{\mu}_{k,n}^T \bar{\theta}_{c}]}{\sum\limits_{c \in \mathcal{A}}\exp[\widehat{\mu}_{k,n}^T \theta_c]}\nonumber\\
         &\leq \exp[\widehat{\mu}_{k,n}^T (\theta_a-\bar{\theta}_a+\bar{\theta}_{\max}-\theta_{\min})] \label{eq:inter555}
         \end{align}
Using the definition of $\Delta \rho_{k,n}$ and \eqref{eq:inter555} ,  we impose the following constraint to satisfy the condition  \eqref{eq:theorem22}:
\begin{align}
&\exp[(\mu_{k,n}+\Delta\mu_{k,n})^T (\theta_a-\bar{\theta}_a+\bar{\theta}_{\max}-\theta_{\min})] \nonumber \\
&\leq  \min  (\tilde{D}_1, \widetilde{D}_2
     ,\widetilde{D}_3) + \rho_{k,n}(a)
\end{align}
Hence, we can set
\begin{align}
    &\Delta\mu_{k,n}^T (\theta_a-\bar{\theta}_a+\bar{\theta}_{\max}-\theta_{\min})\nonumber\\
    &\leq \ln \left[ \frac{\min  (\widetilde{D}_1, \widetilde{D}_2
     ,\widetilde{D}_3)+\rho_{k,n}(a)}{\exp[\mu_{k,n}^T \Theta]}  \right ]
\end{align}
which further leads to the following constraint
\begin{align}
\label{eq:inter_col}
    \|\Delta\mu_{k,n}^T\|& \leq \frac{1}{\|\Theta\|}\ln \left[ \frac{\min  (\tilde{D}_1, \widetilde{D}_2
     ,\widetilde{D}_3)+\rho_{k,n}(a)}{\exp[\mu_{k,n}^T \Theta]}  \right ]\nonumber\\
     &\triangleq \widetilde{B}_3
\end{align}
Combining \eqref{eq:theorem21} with the result in \eqref{eq:inter_col} we conclude the proof. 
\end{proof}
\end{appendices}
\bibliographystyle{IEEEtran}
\bibliography{references}
\end{document}